\documentclass{article} 
\usepackage{fullpage}

\def\fighome{.}
\def\bibhome{.}

\usepackage{url}
\usepackage{dsfont}
\usepackage{xspace}
\usepackage{graphicx}
\usepackage{float,pgfplots,wrapfig,sidecap,lipsum}
\usepackage{tabularx}
\usepackage{booktabs}
\usepackage{paralist}
\usepackage{algorithm,algcompatible,lipsum}
\usepackage{algorithmicx}
\usepackage{algpseudocode}
\usepackage{amsfonts}
\usepackage{amsthm}
\usepackage{amsmath}
\usepackage{amssymb}
\usepackage{enumitem}
\usepackage{xcolor}
\usepackage[font=normal,labelfont=bf]{caption}
\usepackage{mathtools} 
\usepackage{multicol}

\usepackage{tablefootnote}
\usepackage{subcaption}
\usepackage{hyperref,psfrag}
\usepackage{tikz}
\usetikzlibrary{fit}
\usetikzlibrary{calc,shapes}
\usetikzlibrary{decorations.pathmorphing} 
\usetikzlibrary{fit}					
\usetikzlibrary{backgrounds}	

\usepackage{bm}
\usepackage{mathrsfs} 
\usepackage{stmaryrd}
\usepackage[utf8]{inputenc}
\usepackage[english]{babel}
\usepackage{wrapfig}

\newcommand*{\mytensor}[1]{\bm{\mathcal{#1}}}
\newcommand*{\mymatrix}[1]{\bm{#1}}
\newcommand*{\myvector}[1]{\bm{#1}}
\newcommand*{\noise}[1]{\widehat{#1}}
\newcommand*{\mm}{\mymatrix}
\newcommand*{\mv}{\myvector}
\newcommand{\mt}{\mytensor}
\newcommand*{\mmCols}[2]{\mm{#1}_{[#2]}}

\DeclareMathOperator*{\argmin}{arg\,min}

\newcommand*{\true}[1]{\bm{#1}}
\newcommand*{\est}[1]{\bm{#1}^*}

\newcommand{\ourfullalgo}{{Slicing Initialized Alternating Subspace Iteration}\xspace}
\newcommand{\ouralgoshort}{{s-ASI}\xspace}

\newcommand{\ourfirstalgo}{{Alternating Subspace Iteration}\xspace}
\newcommand{\ourfirstalgoshort}{ASI\xspace}
\newcommand{\qr}[1]{\textsf{QR}\left(#1\right)}
\newcommand{\nameInit}{Slice-Based Initialization\xspace}

\newcommand{\Ornoisebound}{\min\{O(\frac{(\lambda_r - \lambda_{r+1})\epsilon}{\sqrt{r}}), O(\delta_0\frac{\lambda_r -\lambda_{r+1}}{\sqrt{d}})\}}
\newcommand{\uniformnoisebound}{\min\{O(\frac{\Delta\epsilon}{\sqrt{R}}), O(\delta_0\frac{\Delta}{\sqrt{d}})\}}

\newcommand{\uniformsotabound}{
\min\{
O(\frac{\Delta\epsilon}{\sqrt{R}})
,
O(\delta_0\frac{\Delta^2}{\sqrt{dR}})
))}

\newcommand{\gap}{\mathsf{gap}}
\newcommand{\op}{\mathsf{op}}
\newcommand{\diag}{\mathsf{Diag}}

\newcommand{\identitymatrix}{\mm{I}}
\newcommand\R{\mathbb{R}}

\def\tha{{\mbox{\tiny th}}}
\def\beq{\begin{equation}}

\def\eeq{\end{equation}\noindent}

\newcommand{\bp}{\begin{psfrags}}
\newcommand{\ep}{\end{psfrags}}
\newcommand{\bc}{\begin{center}}
\newcommand{\ec}{\end{center}}

\newcommand{\Diag}{\mathsf{Diag}}

\newcommand{\fro}{\mathsf{F}}

\DeclareMathOperator{\Tr}{\mathsf{trace}}

\newtheorem{theorem}{Theorem}[section]
\newtheorem{lemma}[theorem]{Lemma}
\newtheorem{definition}[theorem]{Definition}
\newtheorem{proposition}[theorem]{Proposition}

\newtheorem{myfact}{Fact}
\newtheorem{remark}[theorem]{Remark}

\newcommand{\mytitle}{Guaranteed Simultaneous Asymmetric Tensor Decomposition via Orthogonalized Alternating Least Squares}

\title{\textbf{\mytitle}}
\date{}
\author{Furong Huang$^*$\\ \emph{Department of Computer Science}\\
\emph{University of Maryland}\\
* \href{mailto:furongh@cs.umd.edu}{furongh@cs.umd.edu} 
   \and Jialin Li \\ \emph{Department of Mathematics}\\
\emph{University of Maryland}\\  \href{mailto:jl233@math.umd.edu}{jl233@math.umd.edu}
   \and Xuchen You \\ \emph{Department of Computer Science}\\
\emph{University of Maryland}\\  \href{mailto:xyou@cs.umd.edu}{xyou@cs.umd.edu}
 }

\begin{document}
\maketitle
\begin{abstract}
Tensor CANDECOMP/PARAFAC (CP) decomposition is an important tool that solves a wide class of machine learning problems.
Existing popular approaches recover components one by one, not necessarily in the order of larger components first. 
Recently developed simultaneous power method obtains only a high probability recovery of top $r$ components even when the observed tensor is noiseless. 
We propose a \emph{\ourfullalgo} (\ouralgoshort) method that is guaranteed to recover top $r$ components ($\epsilon$-close) simultaneously for (a)symmetric tensors almost surely under the noiseless case (with high probability for a bounded noise) using $O(\log(\log \frac{1}{\epsilon}))$ steps of tensor subspace iterations.
Our \ouralgoshort introduces a \nameInit that runs $O(1/\log(\frac{\lambda_r}{\lambda_{r+1}}))$ steps of matrix subspace iterations, where $\lambda_r$ denotes the $r^\tha$ top singular value of the tensor. 
We are the first to provide a theoretical guarantee on simultaneous orthogonal asymmetric tensor decomposition. 
Under the noiseless case, we are the first to provide an \emph{almost sure} theoretical guarantee on simultaneous orthogonal tensor decomposition. 
When tensor is noisy, our algorithm for asymmetric tensor is robust to noise smaller than 
$\Ornoisebound$, where $\delta_0$ is a small constant proportional to the  probability of bad initializations in the noisy setting.
\end{abstract}

\section{Introduction}

Latent variable models are probabilistic models that are versatile in modeling high dimensional complex data with hidden structure.
%
The method of moments~\cite{hall2005generalized} relates the observed data moments with model parameters using a CP tensor decomposition~\cite{kolda2009tensor}.
Specifically, learning latent variable models using the method of moments involves identifying the linearly independent components of a data moment tensor $\mytensor{T}$.  
The assumption of linearly independent components is practical and holds in many applications such as topic model, community detection and recommender systems.  Orthogonal assumption is not stronger than a linear independence one. 
CP decomposition for tensors with linearly independent components can be reduced to CP decomposition for tensors with orthogonal components 
using whitening (a multilinear transformation).
{Orthogonal tensor decomposition is key for spectral algorithms for solving many ML problems. For instance, paper~\cite{huang2015online} discusses how this method outperforms state-of-the-art variational inference in topic modeling and community detection.}
Due to finite number of data examples, we observe a data empirical moment $\widehat{\mytensor{T}}$ (a noisy version of the data moment $\mytensor{T}$):  $\widehat{\mytensor{T}} = \mytensor{T} +\mm{\Phi}$, where $\mm{\Phi}$ is the noise tensor. 
Therefore, the core algorithm needed in learning high-dimensional latent variable models in numerous machine learning applications is to find methods that provide guaranteed recovery of the dominant/top linearly independent components of $\mytensor{T}$ using $\widehat{\mytensor{T}}$.

Consider a 3-order underlying tensor $\mytensor{T}$ with components $\mymatrix{A}$,$\mymatrix{B}$, $\mymatrix{C}$,  then $\mytensor{T}=\sum_{i=1}^R \lambda_i\myvector{a}_i\otimes \myvector{b}_i \otimes \myvector{c}_i$  where $\myvector{a}_i$, $\myvector{b}_i$,  $\myvector{c}_i$ are the columns of $\mymatrix{A}$, $\mymatrix{B}$, $\mymatrix{C}$ respectively.  
If $\mytensor{T}$ is {\em symmetric\/}, it permits a symmetric CP decomposition $\mymatrix{A}=\mymatrix{B}=\mymatrix{C}$. 
If $\mytensor{T}$ is asymmetric, $\mytensor{T}$ must be decomposed via an {\em asymmetric\/} decomposition $\mymatrix{A}\neq\mymatrix{B}\neq\mymatrix{C}$. 

\paragraph{Simultaneous Recovery} Popular tensor decomposition methods recovers components one by one.  Unlike previous schemes based on deflation methods~
\cite{anandkumar2014guaranteed} that recover factors sequentially, our
scheme recovers the components simultaneously when $R$ is unknown.  This is a more practical setting. In numerous machine
learning settings, data is generated in real-time, and sequential
recovery of factors may be inapplicable under such online settings. Prior
work~\cite{wang2017tensor} considers a simultaneous subspace
iteration, but is only limited to symmetric tensors. 

\paragraph{Asymmetric Tensors} The symmetric assumption required
by prior methods is restrictive. In most applications, multi-view models or HMMs in which information is asymmetric along
different modes are needed. 
Decomposition of symmetric tensors is
easier than that of asymmetric ones~\cite{kolda2015symmetric} as the constraints of symmetric
entries vastly reduce the number of
parameters in the CP decomposition problem. There is much prior work
\cite{anandkumar2014tensor,anandkumar2016tensor,goyal2014fourier,sharan2017orthogonalized,wang2017tensor}
on decomposing symmetric tensor with identical components across
modes, all of which require multiple random sampling initializations which inevitably induce convergence of the algorithms, only with high probability.



In this paper, we consider simultaneous top $r$ components recovery of asymmetric tensors with
 unknown $R$ number of orthonormal components.  Our goal is to recover top $r$ components simultaneously almost surely when noiseless.
Our \emph{\ourfullalgo} (\ouralgoshort) uses a tensor subspace iteration method, i.e., orthogonalized alternating least square (o-ALS).  

Related works on matrix-based methods, optimization-based methods and other rank-1 or rank-$r$ tensor decomposition methods are surveyed in detail in Section~\ref{sec:relatedwork}.


\subsection{Summary of Contribution}
\paragraph{Contribution to Asymmetric Tensor Decomposition}
We provide the first guaranteed decomposition algorithm, \ourfullalgo (\ouralgoshort),  for asymmetric tensors with a convergence rate $O(\log \log \frac{1}{\epsilon})$ independent of the rank and dimension. 
Our \ouralgoshort recovers the top $r$ components corresponding to the largest $r$ singular values simultaneously with probability 1 under the noiseless case when $R$ is unknown.  Our \ouralgoshort is robust to noise smaller than $\min\{\frac{\sqrt{2}}{8}\frac{\Delta\epsilon}{\sqrt{R}}$, $\delta_0\frac{\lambda_{r}^2-\lambda_{r+1}^2}{8\|\mv{\lambda}\|}$,  $\delta_0\frac{\Delta}{2\sqrt{d}}\}$, where $\Delta = \min_r \lambda_r - \lambda_{r+1}$ denotes the spectral gap of the tensor, $d$ the dimension and $\delta_0$ a constant proportional to the failure probability of initialization.


\paragraph{Contribution to Symmetric Tensor Decomposition}
Our \nameInit procedure applies to symmetric orthogonal tensor decomposition to (1) provide an initialization that guarantees convergence to top $r$ components almost surely when the tensor is noiseless (in contrast to the state-of-the-art random sampling based initialization method~\cite{wang2017tensor} which leads to convergence with some high probability); (2) improve the robustness of the algorithm by allowing larger noise 
$\uniformnoisebound$, in contrast to the state-of-the-art noise level $\uniformsotabound$ allowed. Here we use the fact that the bound can be loosened by replacing $\lambda_r^2-\lambda_{r+1}^2$ by $\Delta^2$. 

\begin{theorem}[Informal \ouralgoshort Convergence Guarantee] Let a tensor permit an noisy orthogonal CP decomposition form $\widehat{\mytensor{T}} =\sum_{i=1}^R \lambda_i\myvector{a}_i\otimes \myvector{b}_i \otimes
\myvector{c}_i +\mm{\Phi}$, where $\lambda_i$ are in descending order. After running $O(\log(\log\frac{1}{\epsilon}))$ steps of tensor subspace iteration in our \ourfirstalgo (Procedure~\ref{algo:main}), the estimated $i^\tha$ component ${\mv{{a}}}^*_i$ converges to the $i$-th component $\mv{a}_i $ with high probability $\lVert \mv{a}_i - {\mv{{a}}}^*_i\rVert \le \epsilon$ for $\forall 1\le i\le r$ when noise is bounded.
\end{theorem}
Note that the results are identifiable up to sign flip. In contrast to rank-1 methods which are identifiable up to sign flip and column permutation, our \ouralgoshort identifies the top-$r$ components with largest $\lambda_i$.
 In table~\ref{tab:rate_compare}, we compare the convergence rate of our algorithm with existing works. A detailed discussion of related work is in section~\ref{sec:relatedwork}.
Our almost surely convergence result with a quadratic convergence rate is supported by experiments in section~\ref{sec:exp}.

\begin{table}
\centering
\scriptsize
	\begin{tabular}{c|l|l|c|c|c}
	 &  		\multicolumn{2}{c|}{\textbf{Method}}		& \multicolumn{2}{c|}{\textbf{\# of iterations}}		& \textbf{Noise} 	\\
	 & 		{Initialization} & {Iterations}  					&  {Initialization} & {Top $r$ recovery} 					& \textbf{allowed} 	 \\
	 \hline
     \cite{anandkumar2014tensor} 	& 	random & rank-1 power & $O(R\log R)$  & $R\log\log\frac{1}{\epsilon}^*$ &  $O(\frac{\epsilon}{d})$  \\
     
     \cite{anandkumar2014guaranteed} & 	SVD & rank-1 ALS 	& $O(R)$  & $R\log\log\frac{1}{\epsilon}^{*}$ & $O(\frac{\epsilon}{\sqrt{d}})$ \\
     
        \cite{sharan2017orthogonalized} & random & rank-r ALS & $O(R)$  & $R\log\log\frac{1}{\epsilon}^{**}$   & -  \\
		 
\cite{wang2017tensor} & sampling & rank-r power &$O(\log d)$ & $\log\log\frac{1}{\epsilon}^\dagger$ &  $\min\{O(\frac{\epsilon}{\sqrt{R}}), O(\frac{1}{\sqrt{dR}})\} $ \\
         
	 \hline
     \ouralgoshort	 & slice based & rank-r ASI & $O(1)${{$^\ddagger$}} & $\log\log\frac{1}{\epsilon}${{$^\ddagger$}}  & $\min\{O(\frac{\epsilon}{\sqrt{R}}),O(\frac{1}{\sqrt{d}}\}) $   \\
	\end{tabular}
	\caption{Convergence comparison of existing methods for symmetric orthogonal tensor decomposition. Both \cite{sharan2017orthogonalized} and our algorithm allow incoherent tensor decomposition. For simplicity, we hide dependency on spectral gap although our method achieves the existing best spectral gap requirement for top $r$ recovery. Our \ouralgoshort matches the state-of-the-art symmetric convergence rate even in the asymmetric setting -- using a \nameInit with $O(1/\log{\frac{\lambda_r}{\lambda_{r+1}}})$ steps of matrix subspace iteration and $O(\log \log \frac{1}{\epsilon})$ steps of \emph{alternating subspace iteration}(ASI). $^*$ The top $r$ factors can only be determined after all $R$ factors are recovered. $^{**}$ The convergence to top $r$ is obtained only when $R$ is known. {$^\dagger$} A high probability convergence due to the sampling initialization. {$^\ddagger$} The initialization will be deterministically successful in noiseless case.}\label{tab:rate_compare}
\end{table}
\normalsize

\section{Related Work}\label{sec:relatedwork}
\paragraph{Rank-1 methods} Both popular rank-1 power methods~\cite{anandkumar2014tensor,wang2016online} (on  orthogonal symmetric tensors using random initialization and deflation) 
and 
rank-1 ALS~\cite{anandkumar2014guaranteed} (
on incoherent tensors via optimizing individual mode of the factors while fixing all other modes, and alternating between the modes) 
require  recovery of \textbf{all} $R$ components sequentially to determine the top $r$ components.
Therefore the convergence rates are inevitably a factor of $R$ slower than our \ouralgoshort as they recover components sequentially, not necessarily in the order of the largest first. 

%

\paragraph{Rank-$r$ methods} (1) Comparison with rank-$r$ power method. Wang et al.~\cite{wang2017tensor} use subspace iteration and prove
the simultaneous convergence of the top-$k$ singular vectors for
orthogonal symmetric tensors. A sampling-based procedure is
used for initialization.   Their sampling-based initialization inevitably introduces a high probability bound even when the observed data is noiseless.
(2) Comparison with rank-$r$ orthogonal ALS. Convergence of a variant of ALS using QR
decomposition~\cite{sharan2017orthogonalized} with random initialization for symmetric tensors
has been proven to require number of iterations linear in $R$.  
 Their method converges to the top $r$ components only when the rank $R$ is known and $r=R$. Their convergence bound of sequential analysis is found to be loose.

\paragraph{Gradient-based methods} Stochastic gradient descent is used to solve
tensor decomposition problem.  In~\cite{ge2015escaping}, an objective
function for tensor decomposition is proposed where all the local
optima are globally optimal. However, the polynomial convergence rate is slower than the double exponential rate achieved in our paper. 

\paragraph{Matrix-based methods}  \cite{tomasi2006comparison} provides a general survey on some early efforts, most of which are based on reduction to matrix decomposition (including subroutines that solves CP decomposition for two-slice tensors through joint diagonalization(\cite{domanov2014canonical}\cite{roemer2008closed})). 
Our method improves upon the line of work mentioned due to the following reasons.
(a) We a noise-robust algorithm that fast converges to top-r components. In contrast, neither \cite{domanov2014canonical} nor \cite{roemer2008closed} presents a convergence rate analysis or robustness analysis under noise. 
(b) \cite{tomasi2006comparison} also discussed several types of trilinear decomposition methods (TLD), which call matrix decompositive subroutines so their convergence rates are limited to be slower than ours. For others mentioned in \cite{tomasi2006comparison}, our method outperforms them in terms of either convergence rate, memory expense, resistance of over-factoring, or ability of simultaneous recovery of top-$r$ components. 



It is empirically shown in \cite{faber2003recent} that a preliminary version of ALS outperforms a series of trilinear decomposition methods (DTLD, ATLD, SWATLD). Our algorithm outperforms the state-of-the-art ALS method in experiments.

More recent works in this direction include \cite{kuleshov2015tensor} and \cite{pimentel2016simpler}. 
Kuleshov et al \cite{kuleshov2015tensor} proposed a sophisticated way of projection such that the gaps of eigenvalues are preserved with high probability. However there is no guarantee of top $r$ recovery. Matrix-decomposition-based methods in general have a logarithmic convergence rate.

The advantages of our method over the \textbf{eigen-decomposition based methods} are:  (1) We achieve $\log(\log(1/\epsilon))$ convergence rate whereas matrix decomposition has $\log(1/\epsilon)$, to the best of our knowledge. 
(2) We provided an analysis for noise tolerance for (a)symmetric tensors, which is either not allowed or missing in the eigen-decomposition based methods.



\section{Tensor \& Subspace Iteration Preliminaries}

Let $[n]:=\{1,2,\ldots,n\}$. 
For a vector $\myvector{v}$, denote the $i^{\tha}$ element as $v_i$. 
For a matrix $\mymatrix{M}$, denote the $i^{\tha}$ row as $\mymatrix{m}^i$, $j^{\tha}$  column as $\mymatrix{m}_j$, and $(i,j)^{\tha}$ element as $m_{ij}$.  
Denote  the first $r$ columns of matrix $\mymatrix{M}$ as $\mmCols{M}{r}$.
An $n$-order (number of dimensions, a.k.a. modes) tensor, denoted as $\mytensor{T}$, is a multi-dimensional array with $n$ dimensions.
For a 3-order tensor $\mytensor{T}$, its $(i,j,k)^{\tha}$ entry is denoted by $T_{ijk}$. 
A tensor is called \emph{cubical} if every mode is of the same size.  A cubical tensor is called \emph{supersymmetric} (or simply refered as symmetric thereafter) if its elements remain constant under any permutation of the indices.

\paragraph{Tensor product} is also known as outer product. 
For $\myvector{a} \in \R^m, \myvector{b} \in \R^n$ and $\myvector{c} \in \R^p$, $\myvector{a} \otimes \myvector{b} \otimes \myvector{c} $ is a $m \times n \times p$ sized 3-way tensor with $(i,j,k)^\tha$ entry being $a_ib_jc_k, \forall 1 \le i \le m, 1\le j \le n, 1\le k \le p$. 
\paragraph{Multilinear Operation} 
The tensor-vector/matrix multilinear operation of $\mytensor{T}$ and matrices $\mymatrix{A}$, $\mymatrix{B}$, $\mymatrix{C}$ is defined as:
$\mytensor{T}(\mymatrix{A}, \mymatrix{B}, \mymatrix{C})_{ijk} = \sum_{a,b,c}\mytensor{T}_{abc}\mymatrix{A}_{ai}\mymatrix{B}_{bj}\mymatrix{C}_{ck}$.
The tensor-vector multiplication is defined similarly.
\paragraph{Tensor operator norm} The operator norm for tensor $\mytensor{T} \in \mathbb{R}^{d_1\times d_2 \times d_3}$ is defined as \\
$\| \mytensor{T} \|_{\op} = \max\limits_{\myvector{\mu}_i \in \mathbb{R}^{d_i}\backslash{ \{\myvector{0}\}}, i=1,2,3} \frac{| \mytensor{T} (\myvector{\mu}_1, \myvector{\mu}_2,\myvector{\mu}_3) |}{\| \myvector{\mu}_1\| \cdot \|\myvector{\mu}_2\| \cdot \|\myvector{\mu}_3 \|}$.
\paragraph{Matricization} is the process of reordering the elements of an $N$-way
tensor into a matrix.
The mode-$n$ matricization of
a tensor $\mytensor{T} \in \R^{I_1\times I_2\times \ldots \times I_N}$ is denoted by $\mytensor{T}_{(n)}$ and arranges the mode-$n$ fibers ~\cite{kolda2009tensor} to be the columns of the resulting matrix, i.e., the $(i_1, i_2,...,i_N)^\tha$ element of  the tensor maps to the $(i_n, j)^\tha$ element of the matrix, where $j=1+\sum_{k=1,k\neq n}^N (i_k-1)\prod_{m=1,m\neq n}^{k-1} I_m$.
\paragraph{Khatri-rao product} $\mymatrix{A}\! \odot\! \mymatrix{B}\!:=\!
	\begin{bmatrix}
	a_{11} \mymatrix{b}_1 	&\hspace{-1em}\cdots&\hspace{-1em} 	a_{1p} \mymatrix{b}_p \\
	\tiny{\vdots} 			    	&\hspace{-1em}\tiny{\ddots} &\hspace{-1em} 	\tiny{\vdots} \\
	a_{m1} \mymatrix{b}_1 	&\hspace{-1em}\cdots&\hspace{-1em} 	a_{mp} \mymatrix{b}_p 
	\end{bmatrix}$, 
for $\mymatrix{A}$$\in$$\R^{m\times p}$, $\mymatrix{B}$$\in$$\R^{n\times p}$.

\paragraph{Tensor CP decomposition} A tensor $\mytensor{T}\in \R^{d_1\times d_2\times d_3}$ has \emph{CP decomposition} if the tensor could be expressed exactly as a sum of $R$ rank-one components, i.e. $\exists ~ \bm{\Lambda}$, $\bm{A}$, $\bm{B}$, $\bm{C}$ such that $\mytensor{T} = \sum_{i=1}^R \lambda_i \bm{a}_i \otimes \bm{b}_i \otimes \bm{c}_i $, where $R$ is a positive integer, $\bm{\Lambda} = \diag([\lambda_1, \lambda_2, \cdots, \lambda_R])$, $\bm{A} = [\bm{a}_1,\bm{a}_2, \ldots, \bm{a}_R] \in \R^{d_1\times R} $ , $\bm{B} = [\bm{b}_1,\bm{b}_2, \ldots, \bm{b}_R] \in \R^{d_2\times R}$ and $\bm{C} = [\bm{c}_1,\bm{c}_2, \ldots, \bm{c}_R] \in \R^{d_3\times R}$. If so, we donote the CP decomposition as $\mytensor{T}=  \llbracket \bm{\Lambda} ; \bm{A}, \bm{B}, \bm{C} \rrbracket$ and call $\bm{A}$, $\bm{B}$, $\bm{C}$ factors of this CP decomposition. The \emph{rank} of $\mytensor{T}$ is the smallest number of rank-one components that sum to $\mytensor{T}$.

\paragraph{Subspace similarity}
\begin{definition}[Subspace Similarity~\cite{zhu2013angles}]  \label{def:subspace_similarity}
Let $S_1$, $S_2$ be two $m$-dimension proper subspaces in $\mathbb{R}^n$ spanned respectively by columns of two basis matrices $\mymatrix{M}_1, \mymatrix{M}_2$. Let $\mymatrix{M}_2^c$ be the basis matrix for the complement subspace of $S_2$. The principal angle  $\theta$  formed by $S_1$ and $S_2$ is 
$ \cos(\theta) = \min\limits_{\myvector{y} \in \mathbb{R}^m} \frac{\| \mymatrix{M}_1^\top \mymatrix{M}_2 \myvector{y} \| }{\| \mymatrix{M}_2 \myvector{y} \|} = \sigma_{\text{min}}(\mymatrix{M}_1^\top \mymatrix{M}_2)$,   
$\sin(\theta) = \max\limits_{\myvector{y} \in \mathbb{R}^{n-m}} \frac{\| \mymatrix{M}_1^\top \mymatrix{M}_2^c \myvector{y} \| }{\| \mymatrix{M}_2^c \myvector{y} \|} = \sigma_{\text{max}}(\mymatrix{M}_1^\top \mymatrix{M}_2^c)$,  
$\tan(\theta) = \frac{\sin(\theta)}{\cos(\theta)} = \frac{\sigma_{\text{max}}(\mymatrix{M}_1^\top \mymatrix{M}_2^c)}{\sigma_{\text{min}}(\mymatrix{M}_1^\top \mymatrix{M}_2)}$, where $\sigma_{\text{min}}(\cdot)$ / $\sigma_{\text{max}}(\cdot)$ denotes the smallest / greatest singular value of a matrix.
	\end{definition}

\section{Asymmetric Tensor Decomposition Model}
Consider a rank-$R$ asymmetric tensor 
$\mytensor{T}\in \R^{d\times d\times d}$
with latent factors $\mymatrix{\Lambda}$, $\mymatrix{A}$, $\mymatrix{B}$ and $\mymatrix{C}$
\begin{equation}\label{eq:generative}
	 \mytensor{T}=  \llbracket \mymatrix{\Lambda} ; \mymatrix{A}, \mymatrix{B}, \mymatrix{C} \rrbracket
	 \equiv \sum_{i=1}^R \lambda_i \myvector{a}_i \otimes \myvector{b}_i \otimes \myvector{c}_i 
\end{equation}
where $\mymatrix{\Lambda}$$=$$\diag([\lambda_1, \cdots, \lambda_R])$, 
$\mymatrix{A}$$=$$[\myvector{a}_1, \ldots, \myvector{a}_R] \in \R^{d\times R}$ and $\mm{A}^\top \mm{A}$$=$${\identitymatrix}$ (similarly for $\mymatrix{B}$, $\mymatrix{C}$).
Without loss of generality, we assume $\lambda_1 > \lambda_2 > \cdots > \lambda_R >0$. Our analysis applies to general order-$n$ symmetric and asymmetric tensors. In this paper, $\mm{A}, \mm{B}, \mm{C}$ are all orthonormal matrices (can be generalized to linearly independent components), and therefore the tensor we find CP decomposition on has a unique orthogonal decomposition, based on Kruskal's condition \cite{kruskal1977three}. 

\textbf{Orthogonal Constraints.}
Although we restrict our discussion to orthogonal CP decompositions, our method applies to a more general setting of linearly independent components. A conventional technique called whitening  can be used to construct an orthogonal tensor without loss of information compared to the original tensor with linearly independent components, a practical setting for various machine learning problems such as topic modeling and community detection. 

Our goal is to discover a CP decomposition with $R$ orthogonal components that best approximates the observed $\widehat{\mytensor{T}}$, and it can be formulated as solving the following optimization problem:
\begin{align}
\argmin_{\!\est{\Lambda}\!, \!\est{A}\!,\!\est{B}\!,\!\est{C}\!}
    \left\lVert \widehat{\mytensor{T}}\! -\! \llbracket \est{\Lambda} ; \est{A}, \est{B}, \est{C} \rrbracket \right\rVert_{\mathsf{F}}^2 
    { \text{s.t. } }
     \Lambda^*_{i,j}\!=\!0, \forall i\!\neq \!j,\!\mymatrix{A}^{*\top}\!\!\! \est{A}\! =\!{\identitymatrix}\!,\! \mymatrix{B}^{*\top}\!\!\! \est{B} \!=\!\identitymatrix\!, \!\mymatrix{C}^{*\top}\!\!\! \est{C} \!=\!\identitymatrix\!
    \label{eqn:objective}
\end{align}
We denote the estimated singular values and factor matrices as $\est{\Lambda}$, $\est{A}$, $\est{B}$ and $\est{C}$ respectively.

\subsection{Difficulty of Asymmetric Tensor Decomposition}
Asymmetric tensor decomposition is more difficult than symmetric tensor decomposition due to the following reasons: (1) the number of parameters required to be estimated is a factor of the tensor order more than the symmetric tensor decomposition (2) the missing symmetry imposes additional difficulty for simultaneous recovery of top-$r$ components of the tensor.

\paragraph{Symmetrization Instability} Existing works often assume that an asymmetric tensor can be symmetrized by a multilinear operation, i.e., $\mt{T}(\mm{M}_a, \mm{M}_b, \identitymatrix)$ becomes symmetric, and thus only prove convergence of symmetric tensor decomposition. Here the symmetrization matrices $\mm{M}_a = \mt{T}(\mv{b},\identitymatrix,\identitymatrix)^\top$ $\mt{T}(\identitymatrix,\identitymatrix,\mv{a})^{-1}$ and $\mm{M}_b = \mt{T}(\identitymatrix,\mv{b},\identitymatrix)^{\top}$ $(\mt{T}{(\identitymatrix,\identitymatrix,\mv{a})^{\top}})^{-1}$ with $\mv{a}$ and $\mv{b}$ sampled from a unit sphere. For a proof of the symmetrization, see Appendix~\ref{app:symmetrization}.
However, the computation of $\mm{M}_a$ and $\mm{M}_b$ can be unstable due to the inversion of $\mt{T}(\identitymatrix,\identitymatrix,\mv{a})^{-1}$. Specifically, the inversion of $\mt{T}(\identitymatrix,\identitymatrix,\mv{a})$ can be ill-conditioned, i.e., the condition number $\kappa(\mt{T}(\identitymatrix,\identitymatrix,\mv{a})) = \frac{\max_i{\lambda_i (\mv{a}^\top \mv{c}_i)}}{\min_i{\lambda_i (\mv{a}^\top \mv{c}_i)}}$ can be high. 
Therefore, we consider the direct asymmetric tensor decomposition.

\section{Simultaneous Asymmetric Tensor Decomposition}
One way to solve the trilinear optimization problem in
Equation~\eqref{eqn:objective} is through the alternating least square
(ALS) method \cite{carroll1970analysis, harshman1970foundations,
  kolda2009tensor}. 
The ALS (without orthognalization) approach fixes $\mm{B},\mm{C}$ to compute a closed form solution for $\mm{A}$, then fixes $\mm{A},\mm{C}$ for $\mm{B}$, and fixes $\mm{A},\mm{B}$  for $\mm{C}$. The alternating updates are repeated until the convergence criterions are satisfied. Fixing all but one factor matrix, the problem reduces to a linear least-squares problem over the matricized tensor
\begin{equation}{
\argmin\limits_{\est{A},\est{\Lambda}} \lVert \widehat{\mytensor{T}}_{(1)} -\est{A} \est{\Lambda} (\est{C} \odot \est{B})^\top \rVert_\fro^2,}
\end{equation}
where there exists a closed form solution $\est{A}\est{\Lambda} = \widehat{\mytensor{T}}_{(1)}[(\est{C}\odot \est{B})^\top ]^\dagger$, using the pseudo-inverse. 
ALS converges quickly and is usually robust to noise in
practice. However, the convergence theory of ALS for asymmetric tensor is not well understood.  We fill the gap in this paper by introducing an
\emph{alternating subspace iteration} (ASI) as shown in Algorithm~\ref{algo:main}, for 
asymmetric tensors.

We provide the convergence rate proof of our \ouralgoshort for asymmetric tensor using two steps. \textbf{(1)} Under some \emph{$r$-sufficient initialization condition} (defined in Definition~\ref{def:init_cond}),  we prove an $O(\log(\log(\frac{1}{\epsilon})))$ convergence rate of ASI (Algorithm~\ref{algo:main}). \textbf{(2)} We propose a \emph{\nameInit} (Algorithm~\ref{algo:init}), and prove that after  $O({1}/{\log{\frac{\lambda_r}{\lambda_{r+1}}}})$ steps of matrix subspace iteration, $r$-sufficient initialization condition is satisfied.  We call our algorithm \emph{\ourfullalgo} (\ouralgoshort).

\subsection{ASI under $r$-sufficient Initialization Condition}
 \label{sec:als}
 We define the sufficient initialization condition in Definition~\ref{def:init_cond} under which our \ourfirstalgo algorithm is guaranteed to converge to the true factors of the tensor $\mytensor{T}$.

\begin{definition}[$r$-Sufficient Initialization Condition]\label{def:init_cond}
The $r$-sufficient initialization condition is satisfied if  
$\tan\left(\mmCols{A}{r}, \mymatrix{Q}_{\mmCols{A}{r}}^{(0)}\right) < 1$, $\tan\left(\mmCols{B}{r}, \mymatrix{Q}_{\mmCols{B}{r}}^{(0)}\right) < 1$, and $\tan\left(\mmCols{C}{r}, ~ \bm{Q}_{\mmCols{C}{r}}^{(0)}\right) < 1$. 
\end{definition}

\begin{algorithm}[!htbp]
	\caption{\ourfirstalgo (\ourfirstalgoshort) for Asymmetric Tensor Decomposition}
	\label{algo:main}
	\begin{algorithmic}[1]
		\Require  $d \times d \times d$ sized tensor $\widehat{\mytensor{T}}$, a tentative rank $r$, precision $\epsilon$ 
		\Ensure  $\est{\Lambda}$, $\est{A}, \est{B}, \est{C}$, such that $\lVert \mmCols{A}{r} - \mymatrix{A} ^*\rVert$,  $ \lVert \mmCols{B}{r}  - \est{B}\rVert$, $ \lVert \mmCols{C}{r}  - \est{C}\rVert$ $\le$ $\epsilon$
		\State Initialize $\mathbf{Q}_A^{(0)}, \mathbf{Q}_B^{(0)}, \mathbf{Q}_C^{(0)}$ through Algorithm~\ref{algo:init}
		\For{$k = 0$ \textbf{to} $K=O(\log{( \log{ \frac{1}{\epsilon} } )})$}
		\State  $\mymatrix{Q}_A^{(k+1)} \mymatrix{R}_A^{(k+1)} \leftarrow   \qr{\widehat{\mytensor{T}}_{(1)} (\mymatrix{Q}_C^{(k)}\odot \mymatrix{Q}_B^{(k)})}$  
		\State  $\mymatrix{Q}_B^{(k+1)} \mymatrix{R}_B^{(k+1)} \leftarrow   \qr{\widehat{\mytensor{T}}_{(2)} (\mymatrix{Q}_C^{(k)}\odot \mymatrix{Q}_A^{(k+1)})}$
		\State $\mymatrix{Q}_C^{(k+1)} \mymatrix{R}_C^{(k+1)} \leftarrow   \qr{\widehat{\mytensor{T}}_{(3)} (\mymatrix{Q}_B^{(k+1)}\odot \mymatrix{Q}_A^{(k+1)})}$
		\EndFor
		 \State $(\est{\Lambda}, \est{A}, \est{B}, \est{C})$ $\leftarrow$ Algorithm~\ref{algo:lambda}($\widehat{\mytensor{T}}$, $r$, $\mymatrix{Q}_A^{(K)}$, $\mymatrix{Q}_B^{(K)}$, $\mymatrix{Q}_C^{(K)}$) 
	\end{algorithmic}
\end{algorithm}

Under a satisfaction of the $r$-sufficient initialization condition in
Definition~\ref{def:init_cond}, we update the
components $ \mymatrix{Q}_A^{(k+1)}$, $ \mymatrix{Q}_B^{(k+1)}$ and $
\mymatrix{Q}_C^{(k+1)}$ as in line 3,4,5 of Algorithm~\ref{algo:main}.
We save on expensive matrix inversions over
$(\mymatrix{Q}_C^{(k)}\odot \mymatrix{Q}_B^{(k)})$ as
$(\mymatrix{Q}_C^{(k)}\odot \mymatrix{Q}_B^{(k)}) =
[(\mymatrix{Q}_C^{(k)}\odot \mymatrix{Q}_B^{(k)})^\top]^\dagger$ due
to the orthogonality of $\mymatrix{Q}_B^{(k)}$ and
$\mymatrix{Q}_C^{(k)}$. 
We obtain the following conditional convergence theorem. 
\begin{theorem}[Noiseless Conditional Simultaneous Convergence]\label{thm:main-result}
Under the $r$-sufficient initialization condition in definition~\ref{def:init_cond} and noiseless scenario,  after $K=O(\log(\log(\frac{1}{\epsilon})))$ steps, our \ourfirstalgo in Algorithm~\ref{algo:main} recovers the estimates of the factors $\est{a}_i$, $\est{b}_i$, and $\est{c}_i$ that correspond to the top-$r$ true components with largest $\lambda_i$ up to sign flip, i.e., $\|  \myvector{a}_i -\myvector{a}^*_i \|^2 \le 2\epsilon$, $\forall 1\le i \le r$. Similarly for $ \myvector{b}^*_i$, $\myvector{c}^*_i$, $\forall 1\le i \le r.$
 \end{theorem}
Theorem~\ref{thm:main-result} guarantees that the estimated factors recovered using \ourfirstalgoshort converges to the true factors $\true{A}$, $\true{B}$ and $\true{C}$ when noiseless. We also provided the guarantee for the noisy case in Section~\ref{sec:robust}. 
The convergence rate of \ourfirstalgo is $\log(\log(\frac{1}{\epsilon}))$ when the $r$-sufficient initialization condition is satisfied. The convergence result requires careful manipulation of three different modes. Most ALS methods assume a relaxation to asymmetric tensors, however the existing works only provide convergence results for symmetric tensors. Our work closes the gap between theory and practice.
The proof sketch is in Appendix~\ref{sec:convergence}.
We now propose a novel initialization method in Algorithm~\ref{algo:init} which guarantees that the $r$-Sufficient Initialization Condition is satisfied. 

\begin{algorithm}[!htbp]
	\caption{\nameInit }
	\label{algo:init}
	\begin{algorithmic}[1]
		\Require  
		Tensor $\widehat{\mytensor{T}}$,  $r$
		\Ensure  $\mm{Q}_A^{(0)}, \mm{Q}_B^{(0)}, \mm{Q}_C^{(0)}$
		\State 
		$\mv{e}_i$ $\leftarrow$ $i^\tha$ column of identity matrix
		\If{ $\widehat{\mt{T}}$ is asymmetric}
		\State $\mm{M}^A \leftarrow \sum_{i=1}^d \widehat{\mt{T}}(\identitymatrix, \identitymatrix, \mv{e}_i)\widehat{\mt{T}}(\identitymatrix, \identitymatrix, \mv{e}_i)^\top$
		\State  $\mm{M}^B \leftarrow \sum_{i=1}^d \widehat{\mt{T}}(\mv{e}_i, \identitymatrix, \identitymatrix)\widehat{\mt{T}}(\mv{e}_i, \identitymatrix, \identitymatrix)^\top$
		\State $\mm{M}^C \leftarrow  \sum_{i=1}^d \widehat{\mt{T}}(\mm{I},\mv{e}_i, \mm{I})^\top\widehat{\mt{T}}(\mm{I},\mv{e}_i, \mm{I})$
		\Else			
		\State $\mm{M}^A \leftarrow \widehat{\mt{T}}(\identitymatrix, \identitymatrix, \mv{v}^C)$
		\Comment{$v^C_i = \text{trace}(\widehat{\mytensor{T}}(\identitymatrix, \identitymatrix,\mv{e}_i))$}
		\State 
		$\mm{M}^B \leftarrow \widehat{\mt{T}}(\mv{v}^A, \mm{I}, \mm{I})$  
		 \Comment{$v^A_i = \text{trace}(\widehat{\mytensor{T}}(\mv{e}_i, \identitymatrix,\identitymatrix))$}
		\State $\mm{M}^C \leftarrow \widehat{\mt{T}}(\mm{I},\mv{v}^B, \mm{I})^\top$
		 \Comment{$v^B_i = \text{trace}(\widehat{\mytensor{T}}(\identitymatrix, \mv{e}_i,\identitymatrix))$} 
		\EndIf		
		\State $\mm{Q}_A^{(0)} \leftarrow$ output of Algorithm ~\ref{algo:matsub} on  $\mm{M}^A$ 
		\State $\mm{Q}_B^{(0)}  \leftarrow$ output of Algorithm ~\ref{algo:matsub} on  $\mm{M}^B$
		\State $\mm{Q}_C^{(0)}  \leftarrow$ output of Algorithm ~\ref{algo:matsub} on  $\mm{M}^C$
	\end{algorithmic}
\end{algorithm} 
 \subsection{$r$-Sufficient Initialization: \nameInit$+$Matrix Subspace Iteration}\label{sec:init}
We provide a guaranteed $r$-Sufficient Initialization  $\bm{Q}_{A}^{(0)}, \bm{Q}_{B}^{(0)}, \bm{Q}_{C}^{(0)}$ using a 2-step procedure: 
\begin{itemize}[leftmargin=*,itemsep=0em,topsep=0em]
    \item Prepare matrix $\mm{M}^{{A}}$ ($\mm{M}^{{B}}$, $\mm{M}^{{C}}$) such that the left eigenspace is the column space of $\mm{A}$ ($\mm{B}$, $\mm{C}$). Unlike in \cite{sharan2017orthogonalized} or \cite{wang2017tensor}, Algorithm~\ref{algo:init} recovers  $\mm{M}^{{A}}$ with preserved order of tensor components.
    \item Recover $r$-sufficient $\mm{Q}_A^{(0)}$ (same for $\mm{Q}_B^{(0)}$ and $\mm{Q}_C^{(0)}$) from the matrices above, achieved by Algorithm~\ref{algo:matsub} almost surely in the noiseless case (the discussion of noisy setting is deferred to section~\ref{sec:robust}).
\end{itemize}

We assume a gap between the $r^\tha$ and the $(r+1)^\tha$ singular values for all $r \le R$.
Lemma~\ref{thm:matrixsubspaceiteration} in Appendix~\ref{app:matrixsubspaceiteration}  provides the key intuition behind our initialization procedure. 
Lemma~\ref{thm:matrixsubspaceiteration} shows that given a matrix $\mm{M}\in\R^{d\times d}$, matrix subspace iteration in Algorithm~\ref{algo:matsub} recovers the left eigenspace spanned by eigenvectors of $\mm{M}$ corresponding to $p$ largest eigenvalues.  Therefore, matrix subspace iteration provides insight into how the factors should be initialized. It suggests that as long as we find a matrix whose left eigenspace is the column space of $\true{A}$, we can use matrix subspace iteration to prepare an initialization for \ourfirstalgoshort. 
\begin{algorithm}[!htbp]
	\caption{Matrix Subspace Iteration}
	\label{algo:matsub}
	\begin{algorithmic}[1]
		\Require  
		Matrix $\mm{M}$, $r$
		\Ensure  Left invariant subspace approximation $\mm{Q}^{(J)}$ 
		\State Initialize random orthogonal $\mm{Q}^{(0)} \in \R^{d\times r}$ from \emph{Haar} distribution ~\cite{Mezzadri2006} 
		\For{ $j = 1$ \textbf{to} $J=O(\log(C)/\log(|\frac{\lambda_{r}}{\lambda_{r+1}} |))$}
		\State $\mm{Q}^{(j)} \mm{R}^{(j)}  \leftarrow\qr{ \mymatrix{M}\mymatrix{Q}^{(j-1)}}$
		\EndFor
	\end{algorithmic}
\end{algorithm} 

\begin{theorem}[Noiseless]\label{th:initGuarantee}
Assume that $C\ge 1$ (otherwise $r$-Sufficient Initialization Condition is met after one iteration),  after we run Algorithm~\ref{algo:init} and \ref{algo:matsub} with $J=O(\log(C)/\log(|\frac{\lambda_{r}}{\lambda_{r+1}} |))$ steps,  we guarantee under noiseless scenario, up to sign flip only 
$
\tan\left(\mmCols{A}{r}, \mm{Q}_{A}^{(0)}\right) < 1, \text{ same for } \mm{Q}_{B}^{(0)}\text{ and }\mm{Q}_{C}^{(0)}.
$
\end{theorem}
Theorem~\ref{th:initGuarantee} guarantees that $r$-Sufficient Initialization Condition (Definition~\ref{def:init_cond}) is satisfied after $O(\log(C)/\log(|\frac{\lambda_{r}}{\lambda_{r+1}} |))$ steps of matrix subspace iteration.
The proof of Theorem~\ref{th:initGuarantee} (appendix) follows directly from Lemma~\ref{thm:matrixsubspaceiteration} by setting the convergence tolerance to 1.

\begin{algorithm}[!htbp]
	\caption{Singular Value Computation 
		}\label{algo:lambda}
	\begin{algorithmic}[1]
	\Require  $\widehat{\mytensor{T}}$, $r$, $\mymatrix{Q}_A^{(K)}$, $\mymatrix{Q}_B^{(K)}$, $\mymatrix{Q}_C^{(K)}$ 
	\Ensure  $\mathbf{\Lambda}^*, \est{A}, \est{B}, \est{C}$
	\For{$i = 1$ \textbf{to} $r$}
	\State $\lambda_i^* \leftarrow \widehat{\mytensor{T}}(\est{a}_i, \est{b}_i, \est{c}_i)$
	\Comment{where $\est{a}_i$, $\est{b}_i$, and $\est{c}_i$ denote the $i^
		\tha$ column of  $\mathbf{Q}_A^{(K)}$,  $\mathbf{Q}_B^{(K)}$, $\mathbf{Q}_C^{(K)}$}
	\EndFor
	\State $\est{\Lambda}\leftarrow \Diag(\lambda_1^*, \cdots , \lambda_r^*)$,  $\est{A} \leftarrow  {\mm{Q}_A}^{(K)}$,  $\est{B} \leftarrow  \mm{Q}_B^{(K)}$,  $\est{C} \leftarrow  \mm{Q}_C^{(K)}$ 
\end{algorithmic}
\end{algorithm}
	
\section{\nameInit}
For matrix subspace iteration in Algorithm~\ref{algo:matsub} to work, we prepare a matrix that spans space of eigenvectors of $\mymatrix{A}$ using  \emph{\nameInit} in Algorithm~\ref{algo:init} for symmetric and asymmetric tensors.  matrix subspace iteration is on $ \widehat{\mt{T}}(\identitymatrix, \identitymatrix, \mv{v}^C)$
where $\mv{v}^C_i  = \text{trace}(\widehat{\mytensor{T}}(\identitymatrix, \identitymatrix,\mv{e}_i)) ,  \forall i\in[d]$ for symmetric tensor, and is on $\sum_{i=1}^d \widehat{\mt{T}}(\identitymatrix, \identitymatrix, \mv{e}_i)\widehat{\mt{T}}(\identitymatrix, \identitymatrix, \mv{e}_i)^\top$ for asymmetric tensor. 
\subsection{Performance of \nameInit Algorithm for Symmetric Tensor}\label{para:improv}
Both the performance of symmetric tensor decomposition using rank-1 power
method~\cite{anandkumar2014tensor} and that of simultaneous power
method~\cite{wang2017tensor} will be improved 
using our initialization procedure.
Consider a symmetric tensor with orthogonal components $\mytensor{T} =
\sum_{i = 1}^{R} \lambda_i \mv{u}_i \otimes \mv{u}_i \otimes
\mv{u}_i $ where $\mv{u}_i\perp \mv{u}_j$ and $\mv{u}_i^\top \mv{u}_i=1$.
We start with a vector $\bm{v}^C$ which is the collection of the trace of each third mode slice of tensor $\mytensor{T}$, i.e., the $i^\tha$ element of vector $\mv{v}^C$ is
${v}^C_i =  \sum_{l=1}^{d} \sum_{m=1}^{R} \lambda_m u_{lm}u_{lm} u_{im}  \forall i\in[d].$ 
We then take mode-3 product of tensor $\mytensor{T}$ with the above vector $\mv{v}^C$. As a result, we have Lemma~\ref{thm:initSymmetric}.
%

\textbf{Rank-1 Power Method} with deflation~\cite{anandkumar2014tensor} uses random unit vector initializations, and the power iteration $\bm{v}^{(k+1)} =
	\mytensor{T}(\bm{I}, \bm{v}^{(k)}, \bm{v}^{(k)})$ converges to the tensor
	eigenvector with the largest $| c_i \lambda_i |$ among $| c_1\lambda_1 |,
	\cdots, | c_R \lambda_R |$ where $c_i = \myvector{v}^\top
	\myvector{u}_i$.
A drawback of this property is that random
initialization does {\em not\/} guarantee convergence to the
eigenvector with the largest eigenvalue.

\begin{lemma}[\nameInit improves the rank-1 power method]
For each power iteration loop in rank-1 power method with deflation\cite{anandkumar2014tensor} for symmetric tensors, procedure~\ref{algo:init} guarantees recovery of the eigenvector corresponding to the largest eigenvalue.
\end{lemma}
 \nameInit for symmetric tensors recovers the top-$r$ subspace of the true factor $\true{U}$ as descending order of $\lambda_i^2$ is the same as descending order of $\lambda_i$.
Algorithm~\ref{algo:init} uses $v_k = \text{trace}\big(\mytensor{T}(\bm{I}, \bm{I}, \bm{e}_k)\big)$ and thus 
$\bm{v} 
= \sum_{m=1}^{R} \lambda_m \bm{u}_m.$
Therefore we obtain $c_i = \myvector{v}^\top \myvector{u}_i = \lambda_i$, and the power method converges to the eigenvector $\mv{u}_1$ which corresponds to the largest eigenvalue $\lambda_1$. 

\textbf{Rank-$r$ Simultaneous Power Method} for symmetric tensors is also improved by  Algorithm~\ref{algo:init}.
\begin{lemma}[\nameInit improves the rank-$r$ simultaneous power method] Algorithm~\ref{algo:init} provides an initialization for the matrix subspace iteration used in ~\cite{wang2017tensor} requiring no sampling and averaging, in contrast to $O(\frac{1}{\gamma^2}\log{d})$ steps of
iterations in\cite{wang2017tensor} where $\gamma = \min\limits_{1\le i \le R} \frac{\lambda_i^2 - \lambda_{i+1}^2}{\lambda_i^2}$.
\end{lemma}
In the initialization phase of the algorithm in~\cite{wang2017tensor},
the paper generates random Gaussian vectors $\bm{w}_1, \cdots,
\bm{w}_L \sim \mathcal{N}(\bm{0},\bm{I}_d)$ and let
$\bar{\bm{w}} = \frac{1}{L} \sum_{l=1}^{L}
\mytensor{T}(\bm{I}, \bm{w}_l, \bm{w}_l)$. By doing
$\mytensor{T}(\bm{I}, \bm{I}, \bar{\bm{w}})$,
\cite{wang2017tensor} builds a matrix with approximately squared
eigenvalues and preserved eigengaps. We improve this phase by
simply obtaining vector $\bm{v}$ as $ (\bm{v})_k =
\text{trace}\big(\mytensor{T}(\bm{I}, \bm{I},
\bm{e}_k)\big) $ and substitute $\mytensor{T}(\bm{I}, \bm{I}, \bar{\bm{w}})$ by $\mytensor{T}(\bm{I}, \bm{I}, \bm{v})$.


{Our \nameInit for the symmetric case is slightly different from the asymmetric case for consideration of computational complexity (saving the multiplication of two $d\times d$ matrices). However, the asymmetric \nameInit applies to symmetric case and allows a larger noise. Symmetric \nameInit requires the operator norm of the noise tensor to be 
$O(\delta_0\min\{\frac{\lambda_r^2-\lambda_{r+1}^2}{4\|\mv{\lambda}\|}, \frac{\lambda_r-\lambda_{r+1}}{2d^{(3/4)}}\})$, while the asymmetric \nameInit requires the operator norm of the noise tensor to be
$O(\delta_0\min\{\frac{\lambda_r^2-\lambda_{r+1}^2}{8\|\mv{\lambda}\|}, \frac{\lambda_r-\lambda_{r+1}}{2\sqrt{d}}\})$. }

\subsection{Performance of \nameInit Algorithm for Asymmetric Tensor}\label{para:asymInit}
We provide the first initialization approach for asymmetric tensors, and prove the first convergence result for asymmetric tensors. With our \nameInit which involves a different procedure for asymmetric tensors than for symmetric tensors, the top-$r$ components convergence rate of asymmetric tensors matches that of symmetric tensors.  
Now let us consider the asymmetric tensor $\mytensor{T}$ with orthogonal components $\true{A}$, $\true{B}$ and $\true{C}$. We start with taking the quadratic form of each slice matrix along the third mode of the tensor, i.e., $\mt{T}(\identitymatrix,\identitymatrix,\mv{e}_i)\mt{T}(\identitymatrix,\identitymatrix,\mv{e}_i)^\top$. We obtain
$
\mt{T}(\identitymatrix,\identitymatrix,\mv{e}_i)\mt{T}(\identitymatrix,\identitymatrix,\mv{e}_i)^\top
=\sum_{j=1}^R \lambda_j^2 c_{ij}^2 \mv{a}_j \mv{a}_j^\top$  
 which implies 
$\sum\limits_{i=1}^d \mt{T}(\identitymatrix,\identitymatrix,\mv{e}_i)\mt{T}(\identitymatrix,\identitymatrix,\mv{e}_i)^\top$ 
$=$ $\sum\limits_{j=1}^R \lambda_j^2 \mv{a}_j \mv{a}_j^\top
$
as $\mm{C}$ is orthonormal.
\begin{lemma}[Preserved Component Order]\label{thm:initAsymmetric}
Aggregated quadratic form of slices of asymmetric tensor satisfies	$\sum_{i=1}^d \mt{T}(\identitymatrix,\identitymatrix,\mv{e}_i)\mt{T}(\identitymatrix,\identitymatrix,\mv{e}_i)^\top = \mm{A} \mm{\Lambda} \mm{A}^\top$ where $\mm{\Lambda}$  $=$ $\diag((\lambda_m)_{1\le m\le R})$.
\end{lemma}
Our \nameInit for asymmetric tensors recovers the top-$r$ subspace of the true factors $\true{A}$, $\true{B}$, and $\true{C}$ as the descending order of $\lambda_i^2$ is the same as descending order of $\lambda_i$.

\section{Robustness of the Convergence Result}\label{sec:robust}
We now extend the convergence result to noisy asymmetric
tensors.
For symmetric tensors, there are a number of prior
efforts~\cite{sharan2017orthogonalized,wang2017tensor,anandkumar2016tensor}
showing that their decomposition algorithms are robust to noise. Such
robustness depends on restriction on tensor or structure of the noise such
as low column correlations of factor matrices
(in~\cite{sharan2017orthogonalized}) or symmetry of noise along with
the true tensor (in~\cite{wang2017tensor}).
We provide a robustness theorem of our algorithm under the following \emph{bounded noise} condition.
\begin{definition}[$\delta_0$-bounded Noise Condition] A tensor satisfies the $\delta_0$-bounded noise condition if the noise tensor is bounded in operator norm $\forall 1\leq r\leq R$:
		$
            \|\mm{\Phi}\|_{\op} \leq 
            \min\left\{ \frac{\sqrt{2}}{8}\frac{(\lambda_r - \lambda_{r+1})\epsilon}{\sqrt{r}},\delta_0\frac{\lambda_r^2-\lambda_{r+1}^2}{8\|\mv{\lambda}\|}, \delta_0\frac{\lambda_r -\lambda_{r+1}}{2\sqrt{d}}\right\}.
	$
\end{definition}
Under the bounded noise model, we have the following robustness result. 
\begin{theorem}[\ouralgoshort Convergence Guarantee]\label{thm:robust} Assume the tensor $\mt{T}$ permits a CP decomposition form $\llbracket \mymatrix{\Lambda} ; \mymatrix{A}, \mymatrix{B}, \mymatrix{C} \rrbracket + \mm{\Phi}$ where $\mymatrix{A}, \mymatrix{B}, \mymatrix{C}$ are orthonormal  matrices  and the noise tensor $\mm{\Phi}$ satisfies the $\delta_0$-bounded noise condition. For all $1\leq r\leq R$,
after $J=O(1/\log(|\frac{\lambda_{r}}{\lambda_{r+1}} |))$ matrix subspace iterations in procedure~\ref{algo:matsub} and $O(\log(\log \frac{1}{\epsilon}))$ \ourfirstalgo iterations in procedure~\ref{algo:main}, 
    \ouralgoshort is guaranteed to return estimated
    $\est{\mymatrix{\Lambda}}$,$\est{\mymatrix{A}}$,$\est{\mymatrix{B}}$ and $\est{\mymatrix{C}}$
with probability $> 1 - {O}(\delta_0)$.
   And the estimations satisfy, up to sign flip,  
	$
		\lVert \mv{a}_i - {\mv{{a}}}^*_i\rVert \le \epsilon, ~ \forall 1\le i\le r.$
		Similarly for ${\mv{{b}}}^*_i$ and ${\mv{{c}}}^*_i$ $\forall 1\le i\le r$.
\end{theorem}
The proof follows from the main convergence result~\ref{thm:main} and is in Appendix~\ref{sec:noise}. 
\begin{remark}[1]
  If the goal is to recover all components $\mmCols{A}{R}, \mmCols{B}{R}, \mmCols{C}{R}$, then the preservation of eigenvalue order is not required. Thus the bound on the operator norm of the noise tensor can be relaxed to 
  $O(\frac{\lambda_{\min}}{\sqrt{d}}\epsilon)$,
\end{remark}
\begin{remark}[2]
    For the robustness theorem the worst case is considered (rather than considering the average case associated with a specific family of noise distribution), without any structural assumption. 
In the general case, the noise can be ``malicious'' if there is a sharp
    angle between subspace of $\mm{\Phi}$ and subspace of
$\mytensor{T}$ for every modes.  
\end{remark}

\section{Experiments}\label{sec:exp}
Our method is general enough to be applied as core algorithms for many real-world applications (see~\cite{anandkumar2014tensor,wang2017tensor,huang2015online} for empirical successes). We are not sacrificing any generality by testing on synthetic data. Experimentally we justified convergence theorem~\ref{thm:robust}. Each setting is run 100 times and the [5 percentile, 95 percentile] plots are shown in the figures. 

\subsection{Baseline} 
We compare against the state-of-the-art baseline for asymmetric tensor decomposition, \emph{randomly initialized orthogonalized ALS}(r-ALS)~\cite{sharan2017orthogonalized} ({It is shown in~\cite{faber2003recent} that a preliminary version of ALS empirically outperforms a series of trilinear decomposition methods (DTLD, ATLD, SWATLD). Therefore we choose OALS, the state-of-the-art ALS, as our baseline}), and the state-of-the-art baseline for symmetric tensor decomposition, \emph{simultaneous power iteration}(SPI)~\cite{wang2017tensor}. 

We see scenarios as shown in Figure~\ref{fig:smaller_r_asym} that our \nameInit correctly recovers the top-$r$ components whereas random initialization in r-ALS~\cite{sharan2017orthogonalized} fails to identify the top-$r$ components. 
{We only use randomness for the initialization of matrix subspace iteration which converges to a ``good'' subspace with probability 1, whereas~\cite{sharan2017orthogonalized} randomly initializes OALS. In the noiseless case, we can recover the leading $r$-components with probability 1, which cannot be achieved by~\cite{sharan2017orthogonalized}. 
}
Sampling based initialization in SPI~\cite{wang2017tensor} inevitably introduces a high probability bound even when the observed data is noiseless, and is less robust than our \ouralgoshort under noise as shown in Figure~\ref{fig:smaller_r_sym}. 

\subsection{Symmetric vs Asymmetric}
Figure~\ref{fig:smaller_r_asym} and ~\ref{fig:smaller_r_sym}  illustrate the convergence rate comparison of our \ouralgoshort with the baselines, for asymmetric and  symmetric tensors respectively, when the rank $R$ is unknown to the algorithm. Our \ouralgoshort is always guaranteed to converge for both symmetric and asymmetric tensor, with a much better convergence rate than the baselines.


\begin{figure}[!htbp]
	\begin{minipage}{\textwidth}
		\centering
		\begin{subfigure}[c]{0.32\textwidth}
			\centering
			\includegraphics[width=\textwidth]{\fighome/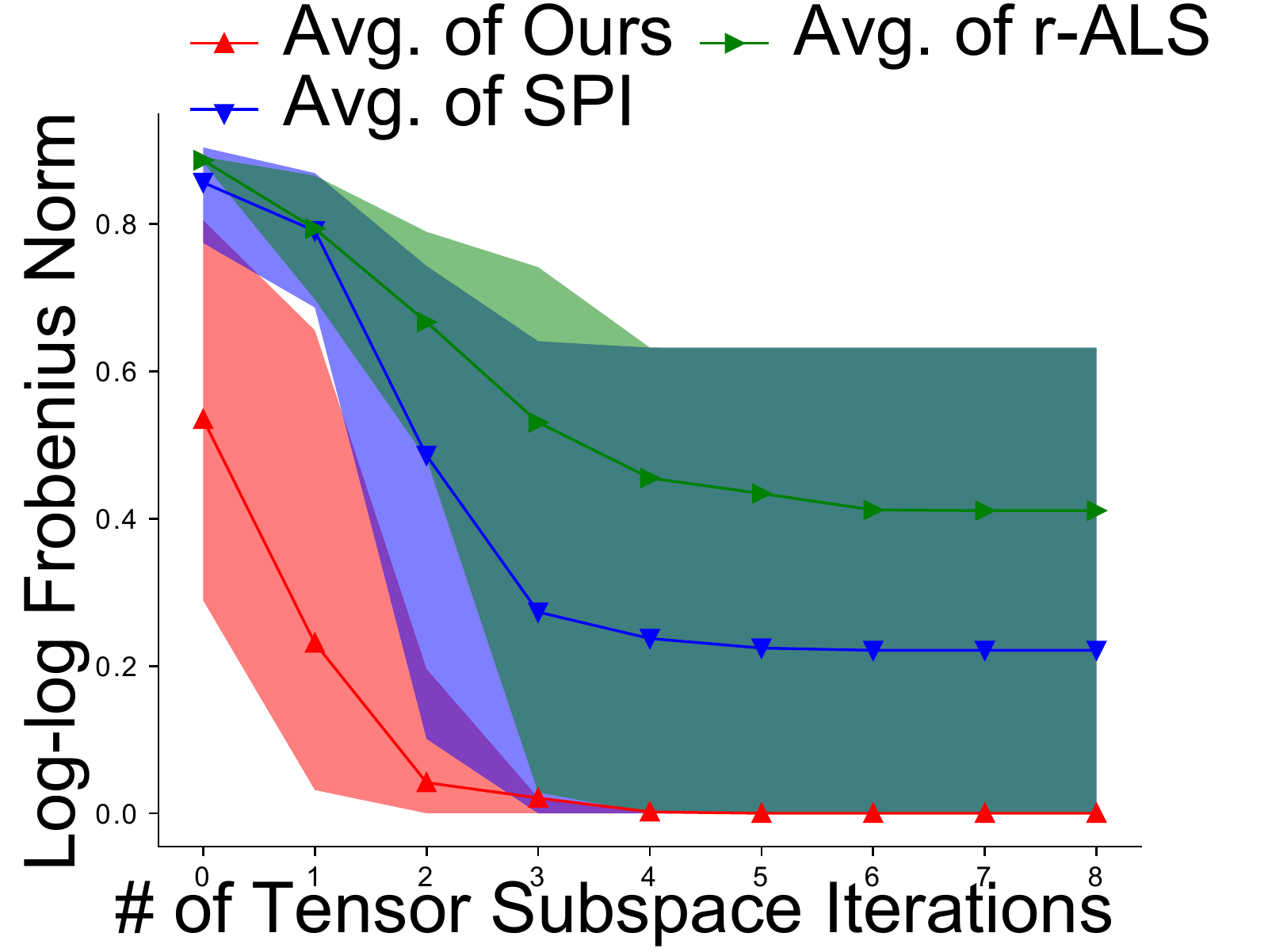}
			\caption{\scriptsize{$\lVert \mmCols{A}{r} - \est{A}\rVert_\fro$}}
			\label{fig:smaller_r_asym_ss1}
		\end{subfigure}
		\hfill
		\begin{subfigure}[c]{0.32\textwidth}
			\centering
			\includegraphics[width=\textwidth]{\fighome/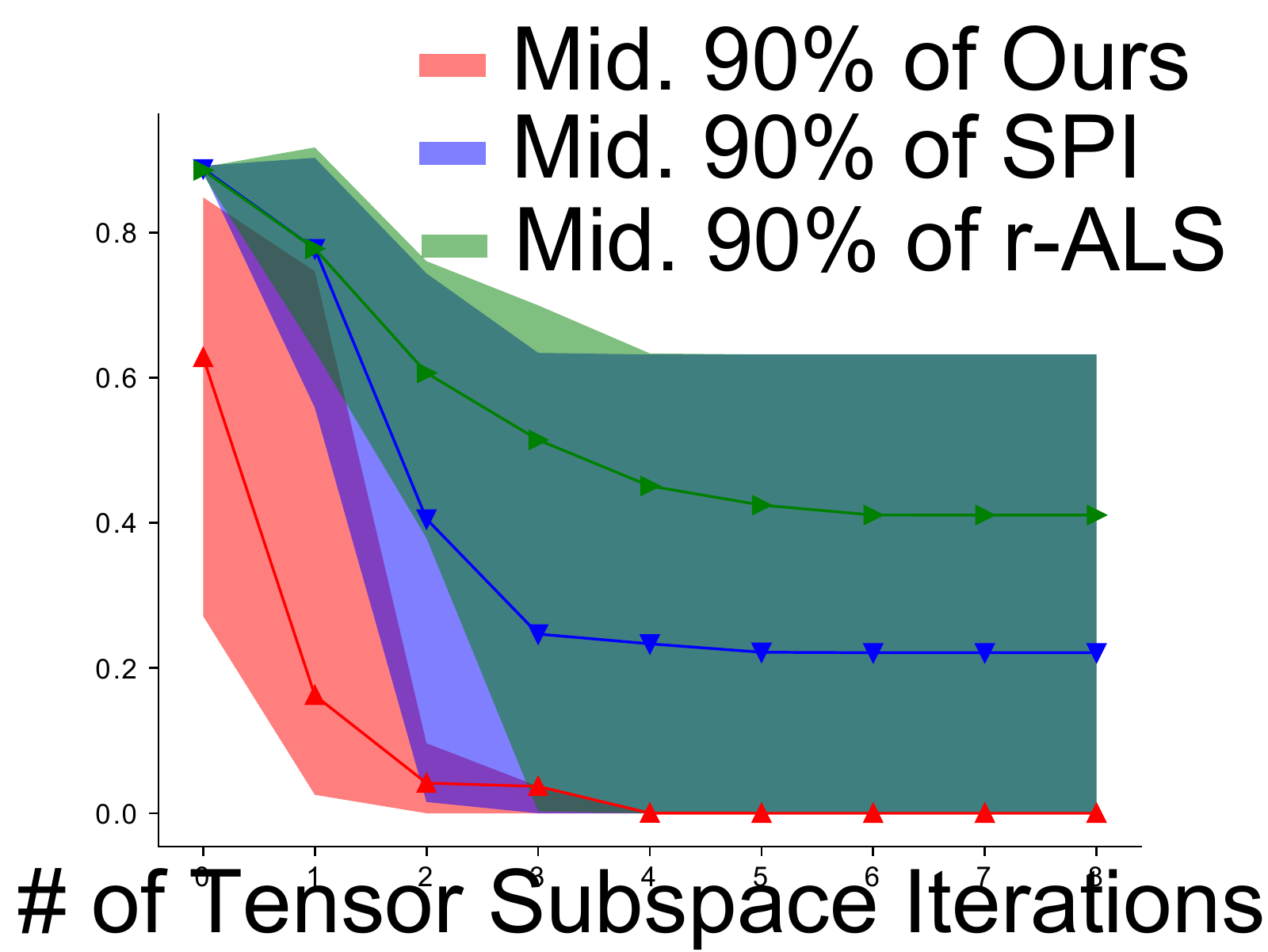}
			\caption{\scriptsize{$\lVert \mmCols{B}{r} - \est{B}\rVert_\fro$}}
			\label{fig:smaller_r_asym_ss2}
		\end{subfigure}
		\hfill
		\begin{subfigure}[c]{0.32\textwidth}
			\centering
			\includegraphics[width=\textwidth]{\fighome/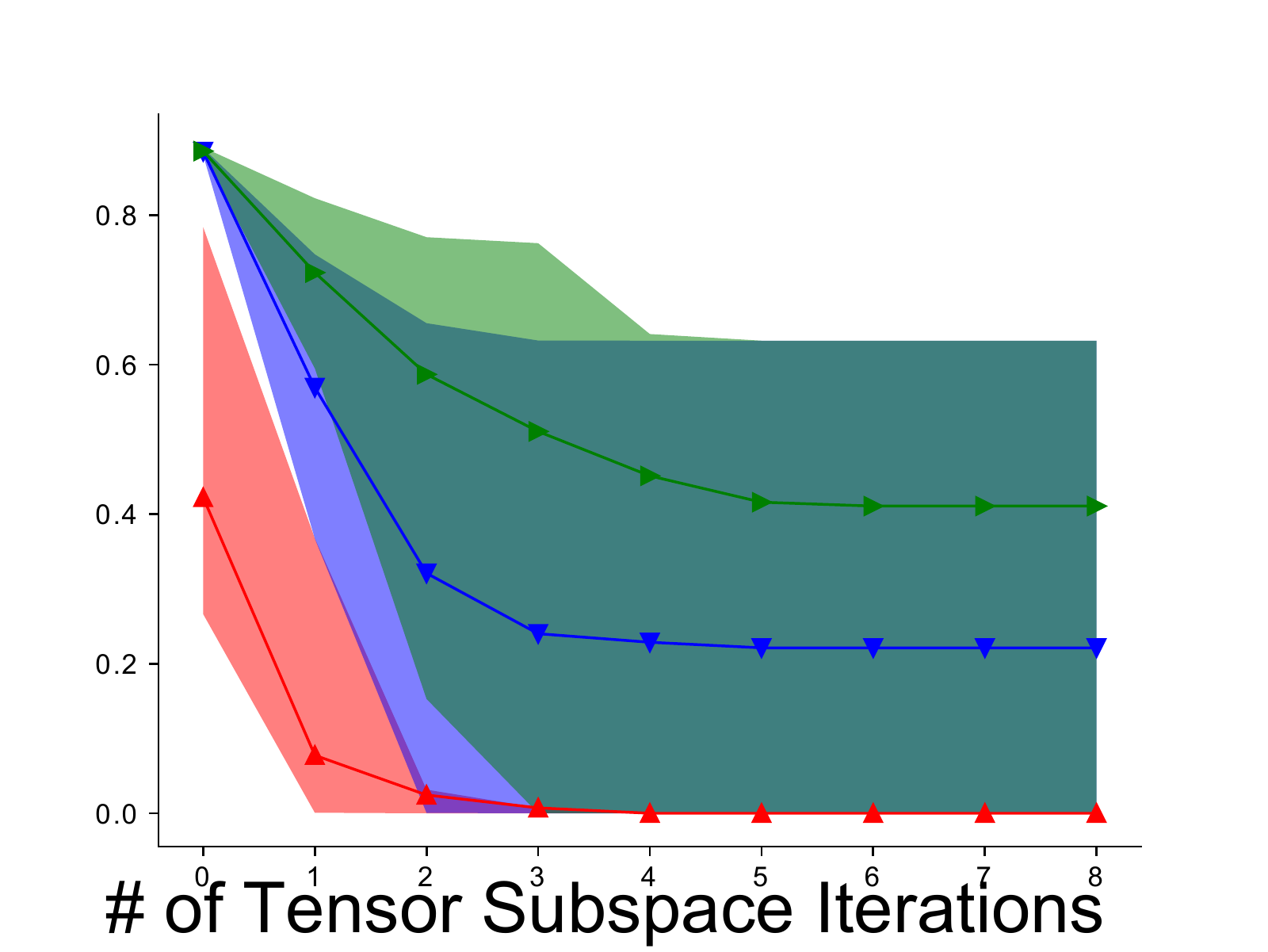}
			\caption{\scriptsize{$\lVert \mmCols{C}{r} - \est{C}\rVert_\fro$}}
			\label{fig:smaller_r_asym_ss3}
		\end{subfigure}
	\end{minipage}
	\begin{minipage}{\textwidth}
		\begin{subfigure}[c]{0.32\textwidth}
			\centering
			\includegraphics[width=\textwidth]{\fighome/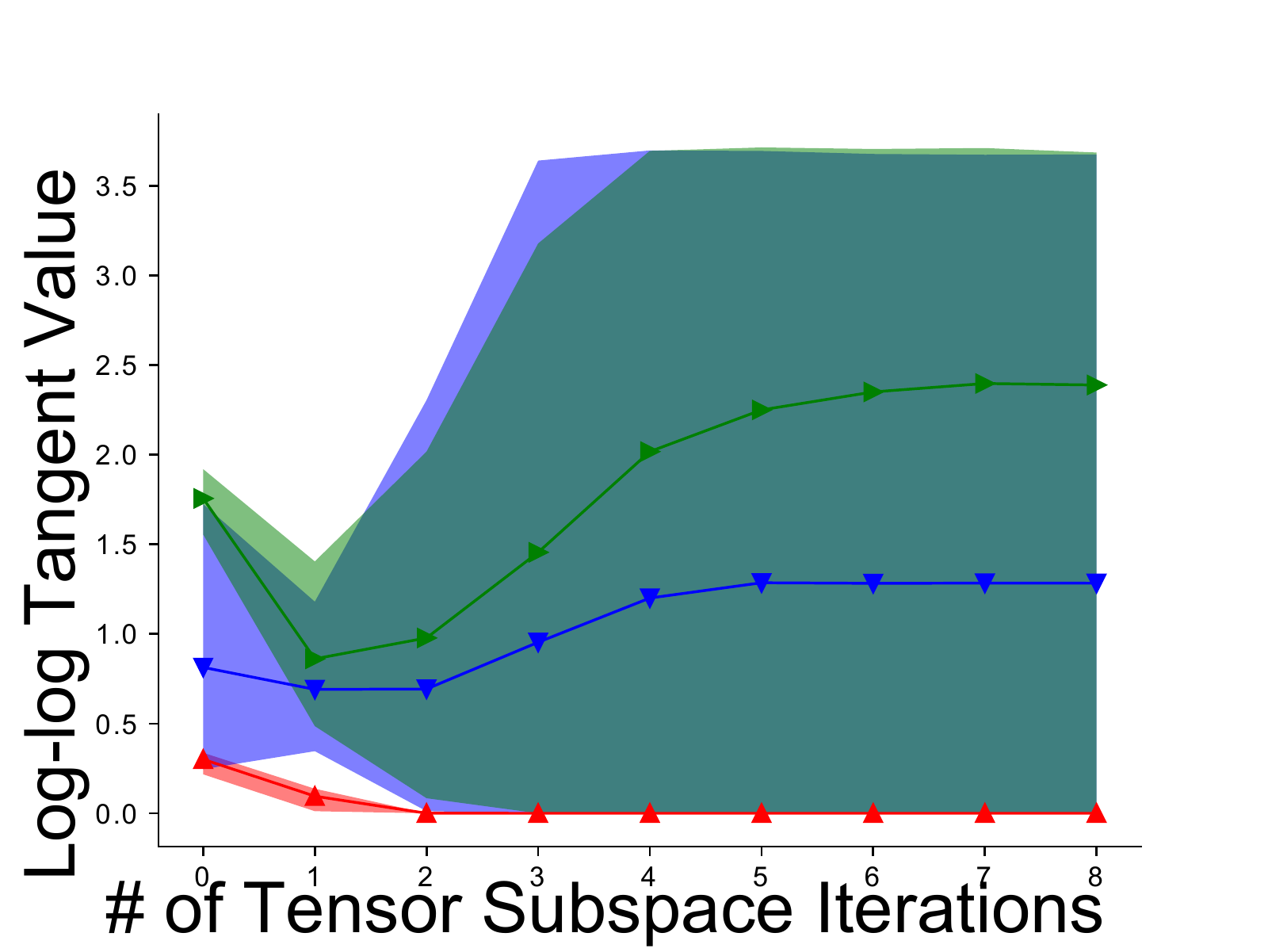}
			\caption{\scriptsize{$\tan{(\mmCols{A}{r}, \est{A})\!}$}}
			\label{fig:smaller_r_asym_tan1}
		\end{subfigure}
		\hfill
		\begin{subfigure}[c]{0.32\textwidth}
			\centering
			\includegraphics[width=\textwidth]{\fighome/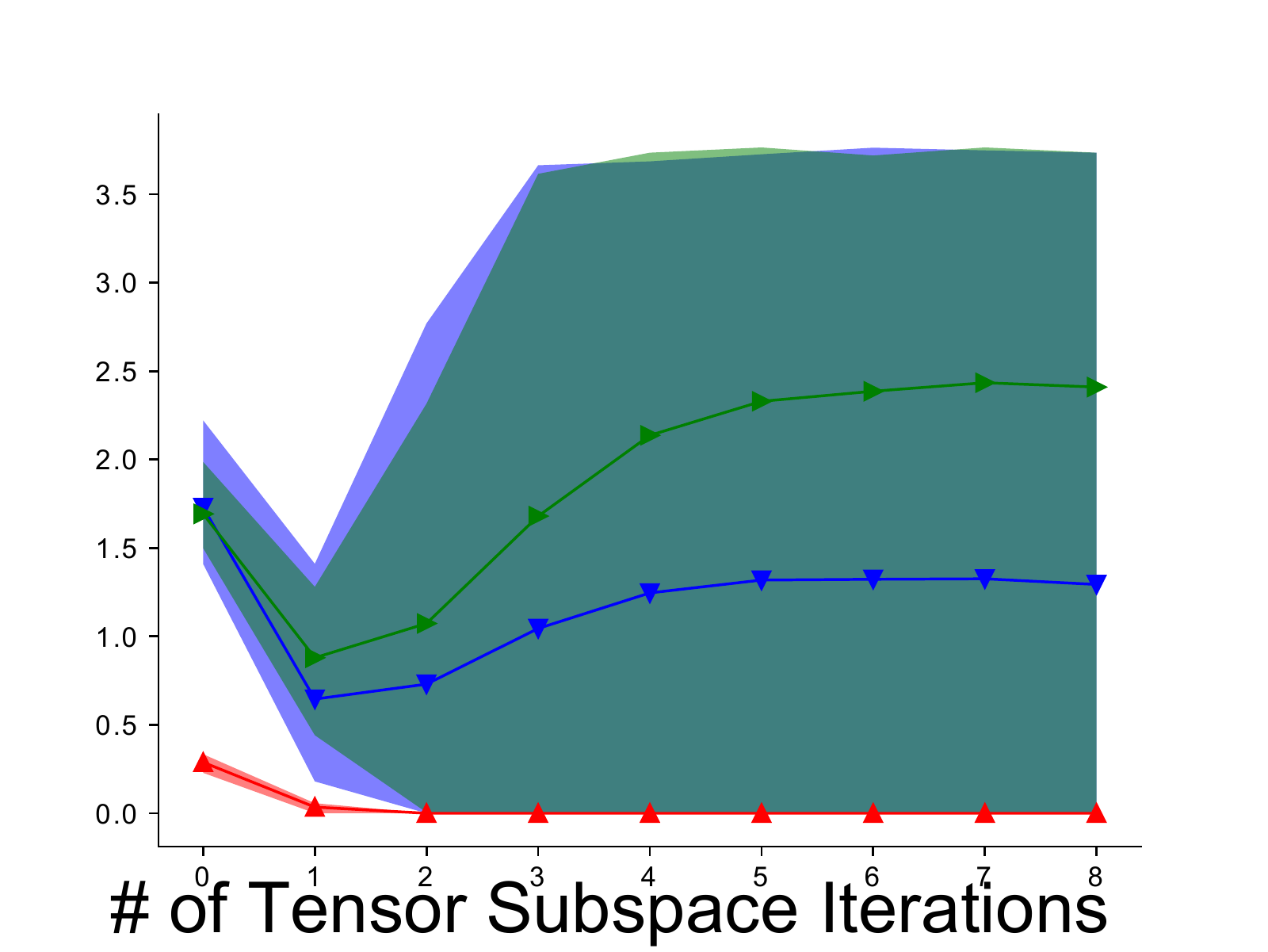}			\caption{\scriptsize{$\tan{(\mmCols{B}{r}, \est{B})}$}}
			\label{fig:smaller_r_asym_tan2}
		\end{subfigure}
		\hfill
		\begin{subfigure}[c]{0.32\textwidth}
			\centering
			\includegraphics[width=\textwidth]{\fighome/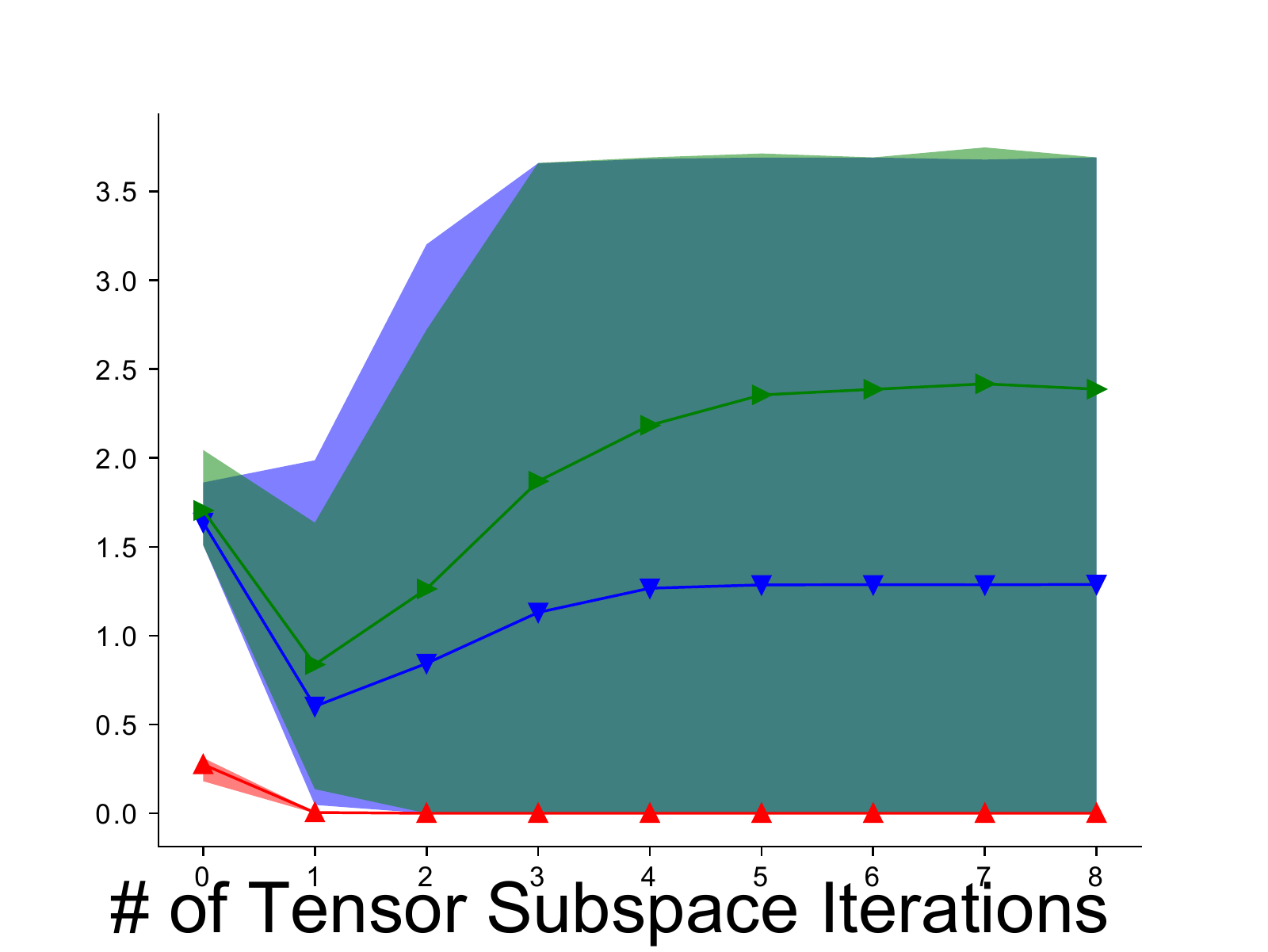}
			\caption{\scriptsize{$\tan{(\mmCols{C}{r}, \est{C})}$}}
			\label{fig:smaller_r_asym_tan3}
		\end{subfigure}
	\end{minipage}
	\caption{Convergence of our \ouralgoshort vs r-ALS~\cite{sharan2017orthogonalized} vs SPI~\cite{wang2017tensor} for asymmetric tensor  when $r=5 < R=10$, $d=500$ run 100 times.
	}
	\label{fig:smaller_r_asym}	
\end{figure}


\begin{figure}[!htbp]
		\centering
		\begin{subfigure}[c]{0.4\textwidth}
			\centering
			\includegraphics[width=\textwidth]{\fighome/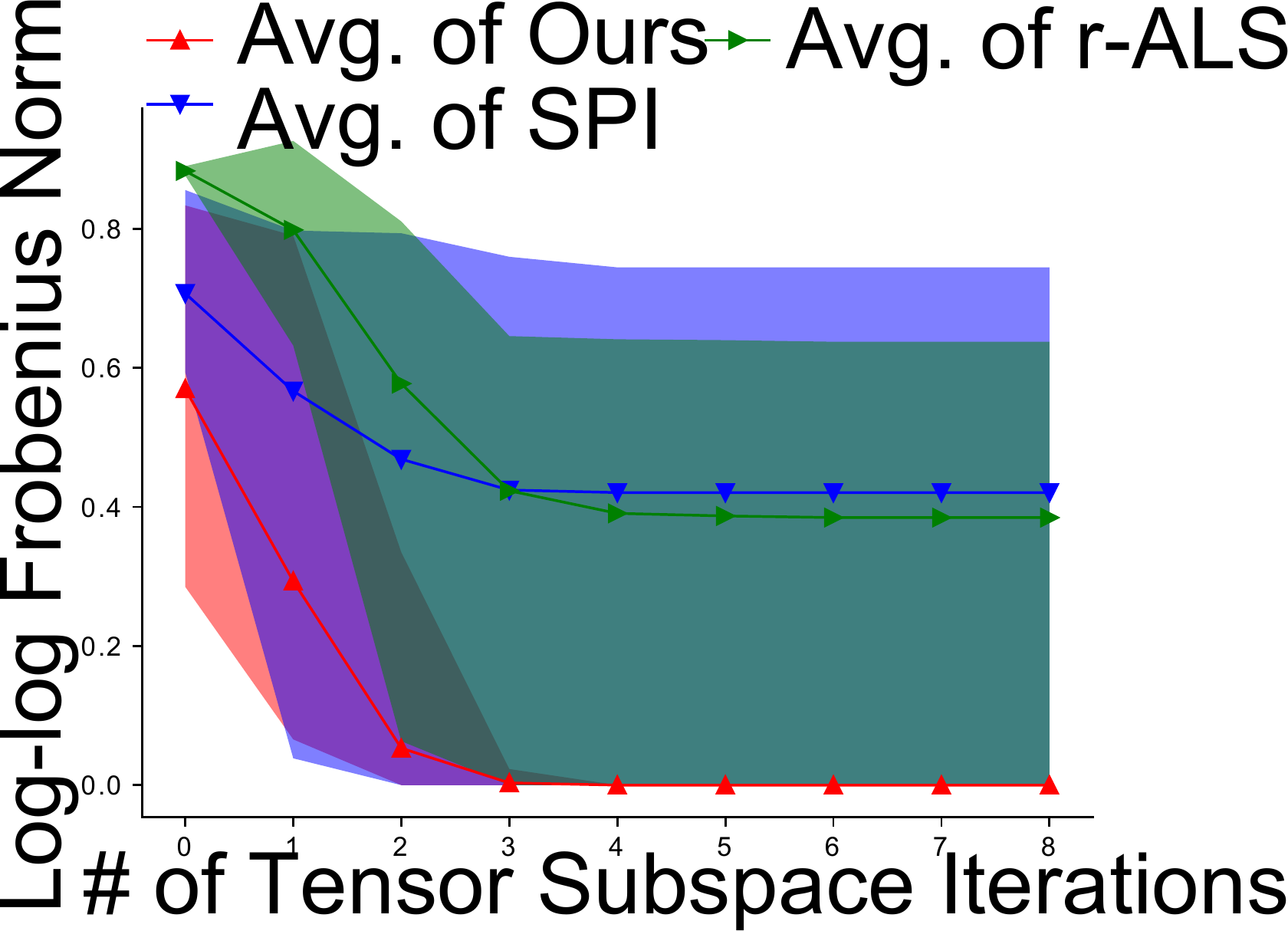}
			\caption{$\lVert \mmCols{A}{r} - \est{A}\rVert_\fro$}
			\label{fig:smaller_r_sym_ss1}
		\end{subfigure}
\begin{minipage}{0.2\textwidth}
\end{minipage}
		\begin{subfigure}[c]{0.4\textwidth}
			\centering
			\includegraphics[width=\textwidth]{\fighome/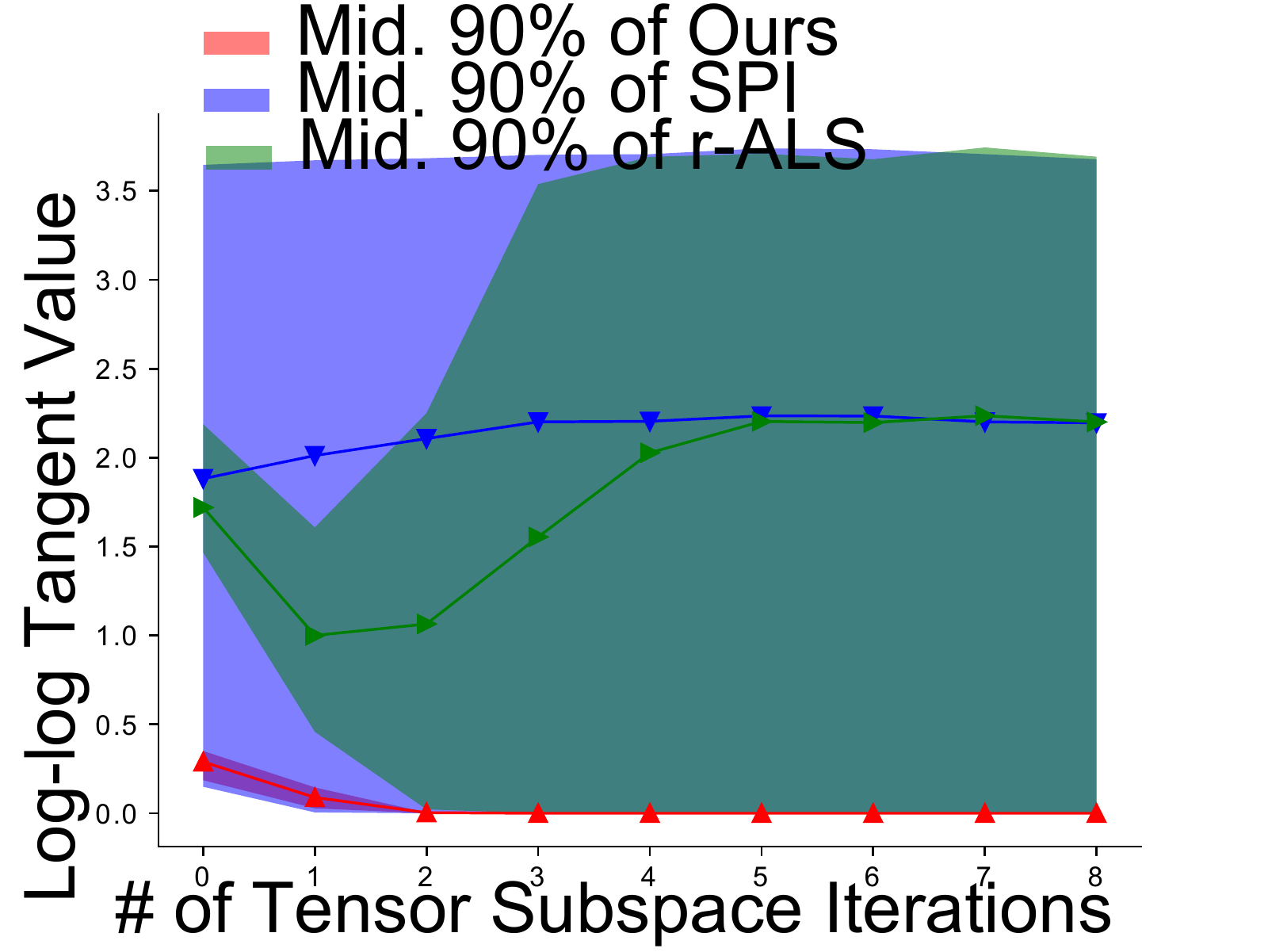}
			\caption{$\tan{(\mmCols{A}{r}, \est{A})}$}
			\label{fig:smaller_r_sym_tan1}
		\end{subfigure}
	\caption{Convergence of our \ouralgoshort vs r-ALS~\cite{sharan2017orthogonalized} vs SPI~\cite{wang2017tensor} for symmetric tensor when  $r=5 < R=10$, $d=500$ run 100 times.
	}
	\label{fig:smaller_r_sym}	
\end{figure}

\subsection{Known vs Unknown Rank.}
 Figure~\ref{fig:smaller_r_asym} illustrates the convergence rate comparison of our \ouralgoshort with the baseline r-ALS when our estimated rank is smaller than the true rank, i.e., $r\le R$. Our
 \ouralgoshort exhibits tremendous advantage when the rank $R$ is unknown to the algorithm.   Figure~\ref{fig:true_r_asym} illustrate the convergence rate comparison of our \ouralgoshort with baselines when our estimated rank is equal to the true rank, i.e., $r= R$.  Our \ourfullalgo achieves better convergence rate than the baselines.


\begin{figure}[!htbp]
	\begin{minipage}{\textwidth}
		\centering
		\begin{subfigure}[c]{0.32\textwidth}
			\centering
			\includegraphics[width=\textwidth]{\fighome/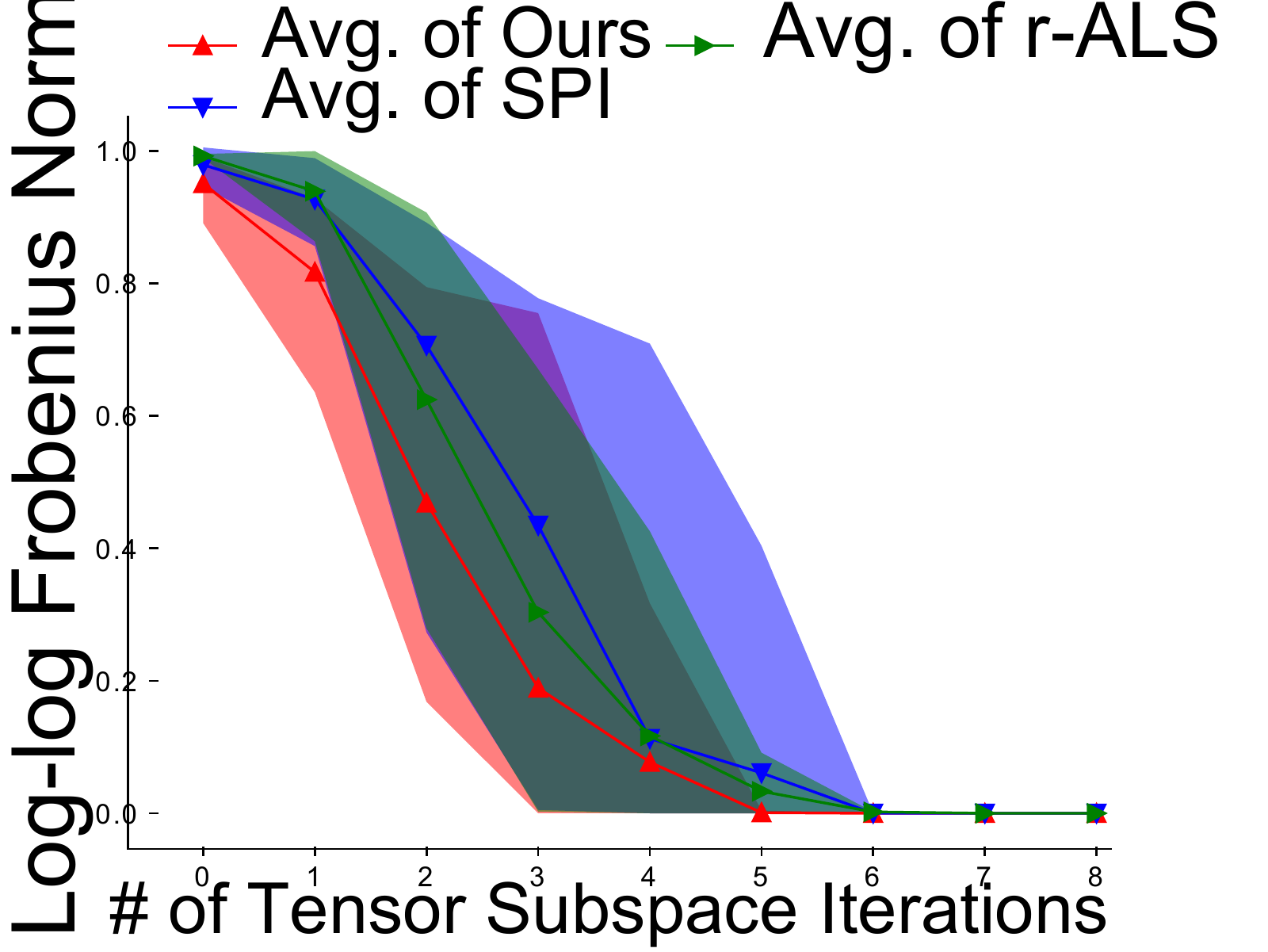}
			\caption{\scriptsize{$\lVert \mmCols{A}{r} - \est{A}\rVert_\fro$}}
			\label{fig:true_r_asym_ss1}
		\end{subfigure}
		\hfill
		\begin{subfigure}[c]{0.32\textwidth}
			\centering
			\includegraphics[width=\textwidth]{\fighome/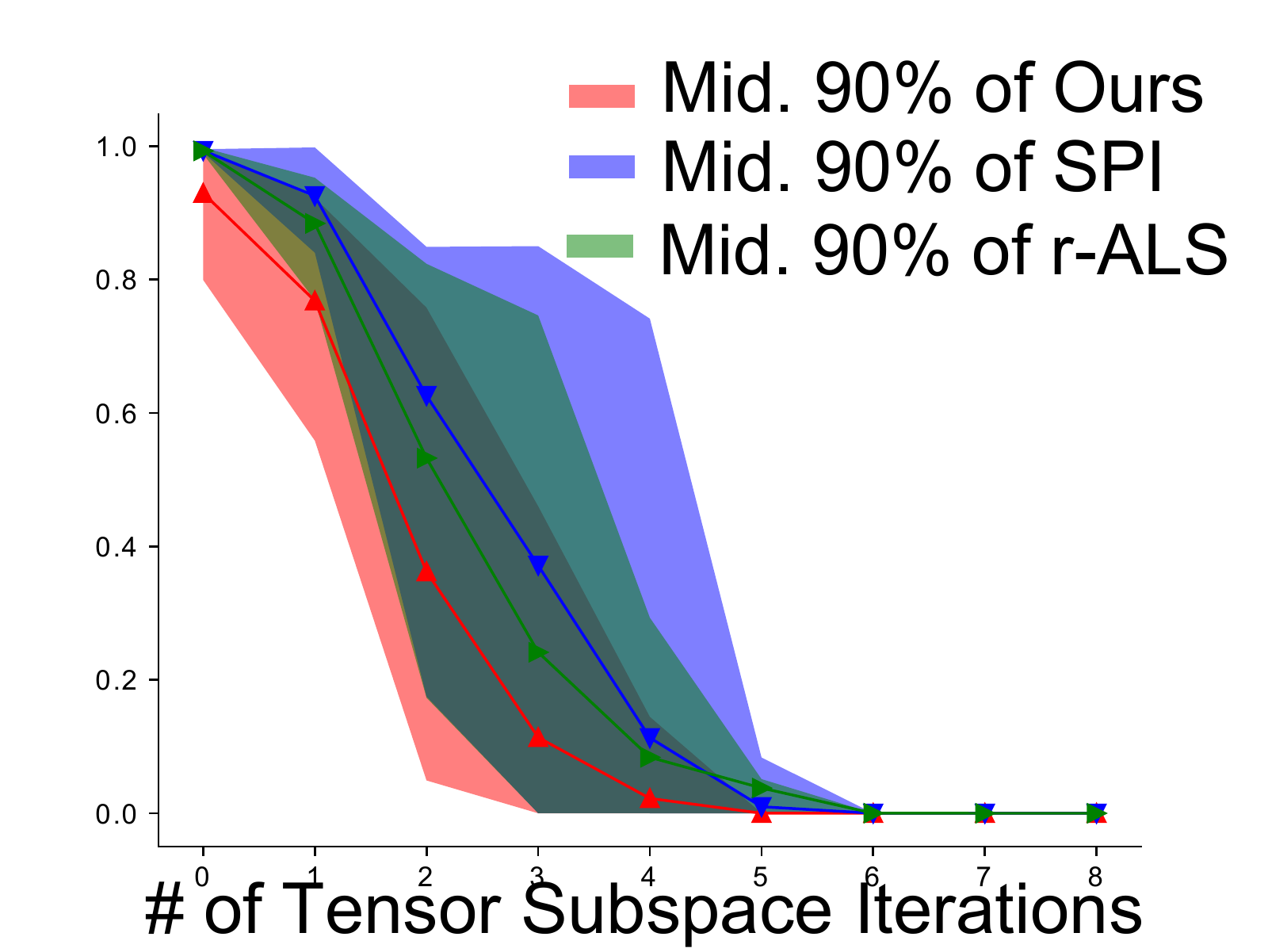}
			\caption{\scriptsize{$\lVert \mmCols{B}{r} - \est{B}\rVert_\fro$}}
			\label{fig:true_r_asym_ss2}
		\end{subfigure}
		\hfill
		\begin{subfigure}[c]{0.32\textwidth}
			\centering
			\includegraphics[width=\textwidth]{\fighome/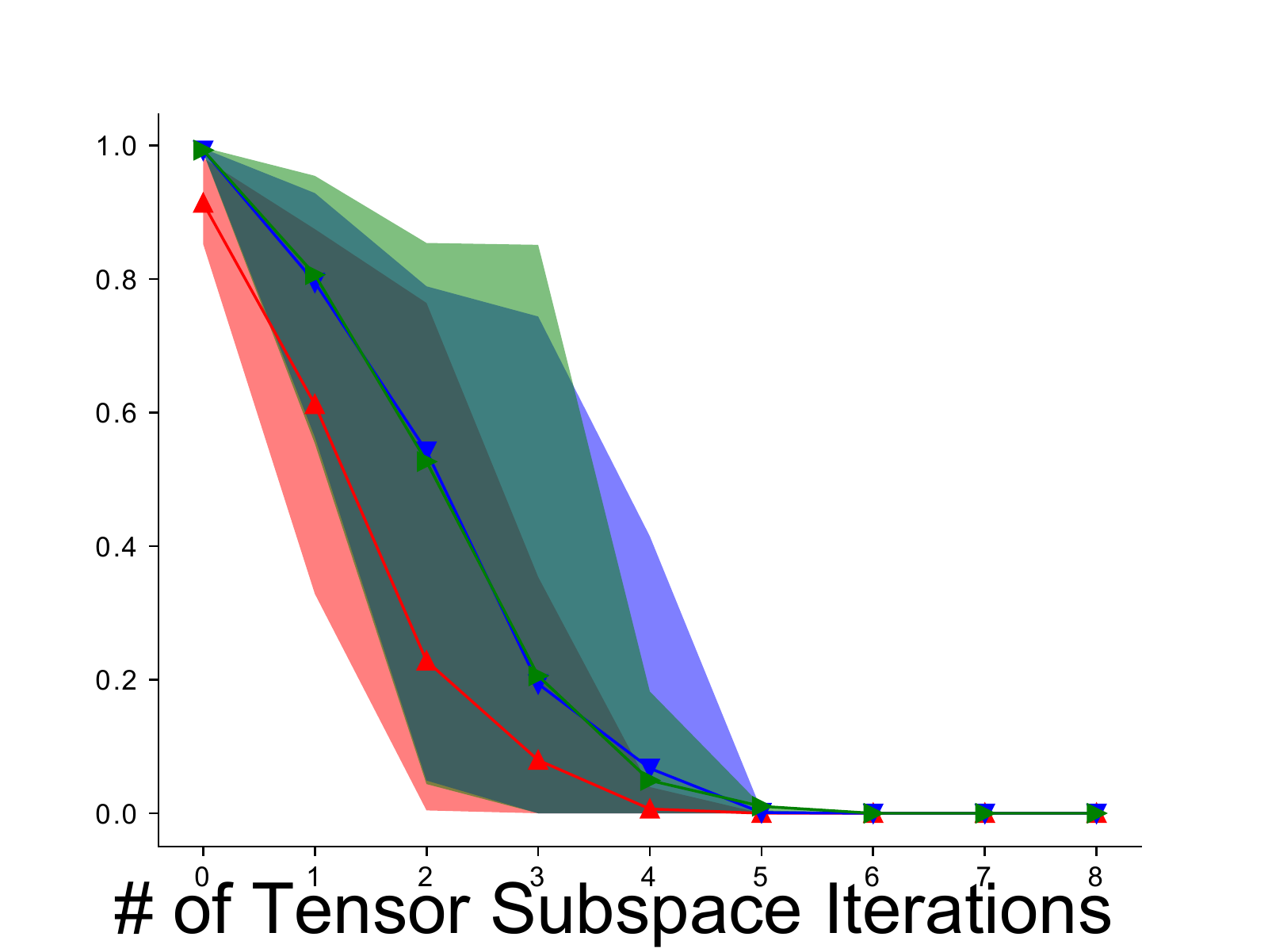}
			\caption{\scriptsize{$\lVert \mmCols{C}{r} - \est{C}\rVert_\fro$}}
			\label{fig:true_r_asym_ss3}
		\end{subfigure}
	\end{minipage}
	\begin{minipage}{\textwidth}
		\begin{subfigure}[c]{0.33\textwidth}
			\centering
			\includegraphics[width=\textwidth]{\fighome/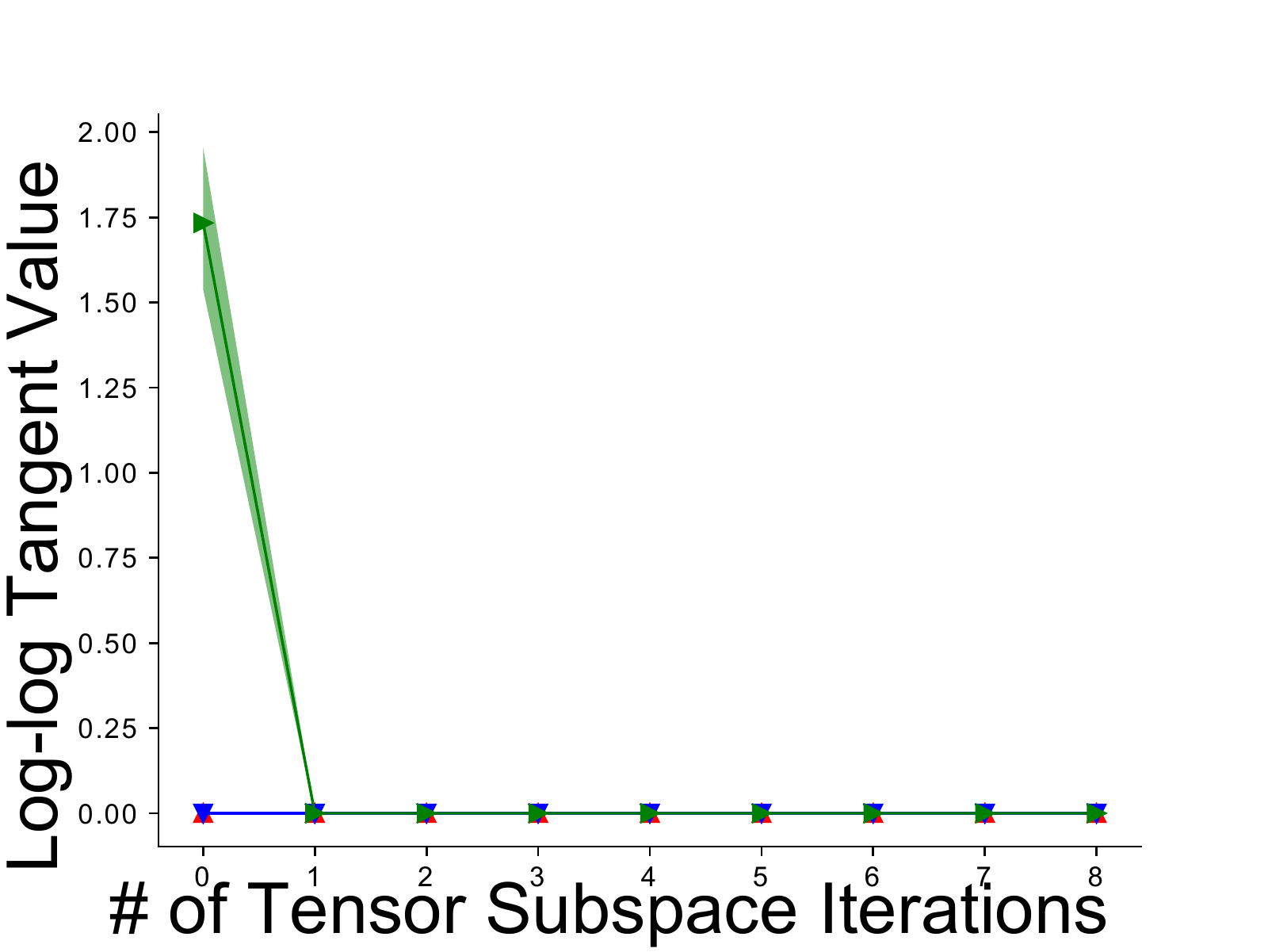}
			\caption{\scriptsize{$\tan{(\mmCols{A}{r}, \est{A})}$}}
			\label{fig:true_r_asym_tan1}
		\end{subfigure}
		\hfill
		\begin{subfigure}[c]{0.32\textwidth}
			\centering
			\includegraphics[width=\textwidth]{\fighome/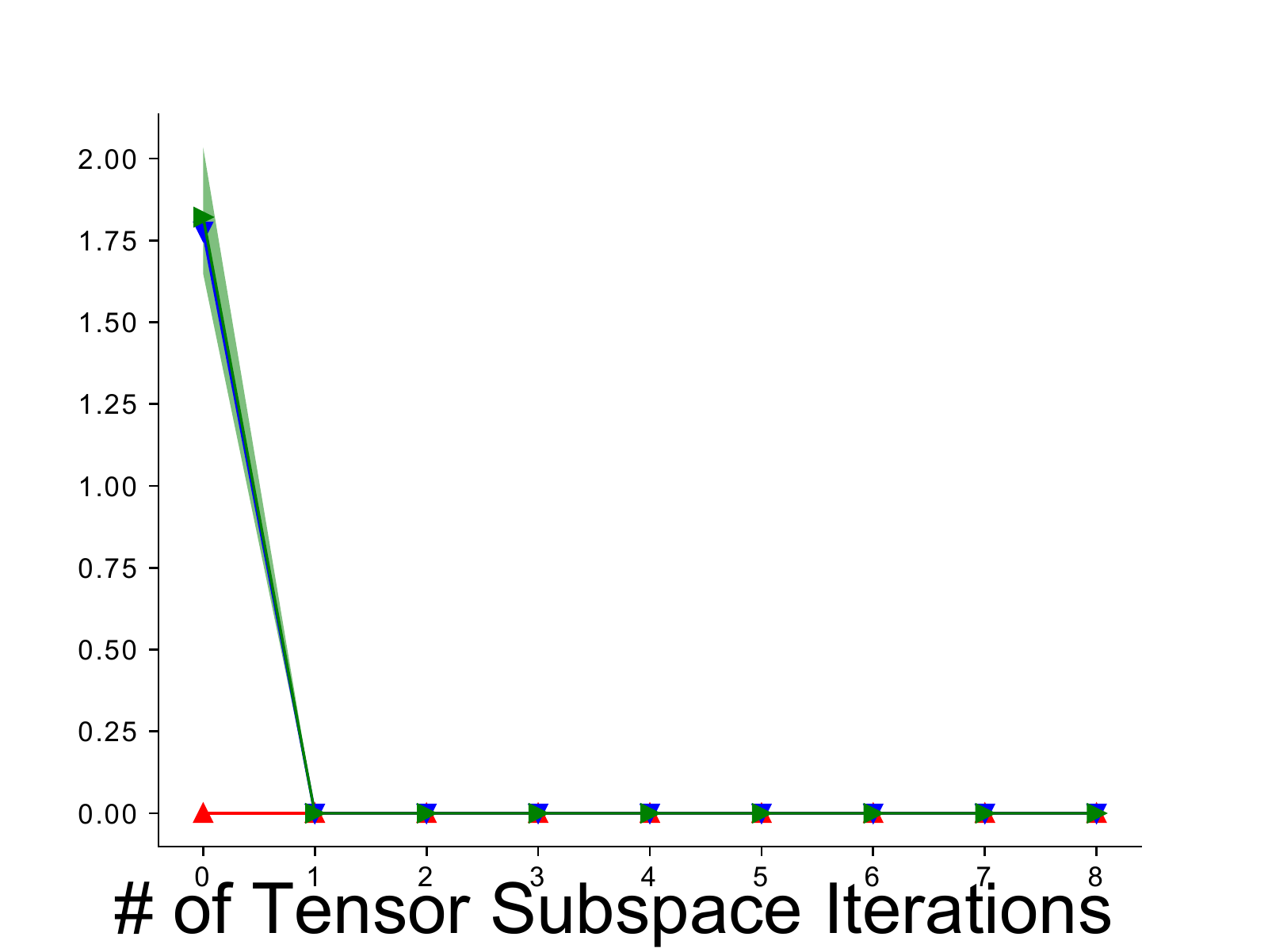}
			\caption{\scriptsize{$\tan{(\mmCols{B}{r}, \est{B})}$}}
			\label{fig:true_r_asym_tan2}
		\end{subfigure}
		\hfill
		\begin{subfigure}[c]{0.32\textwidth}
			\centering
			\includegraphics[width=\textwidth]{\fighome/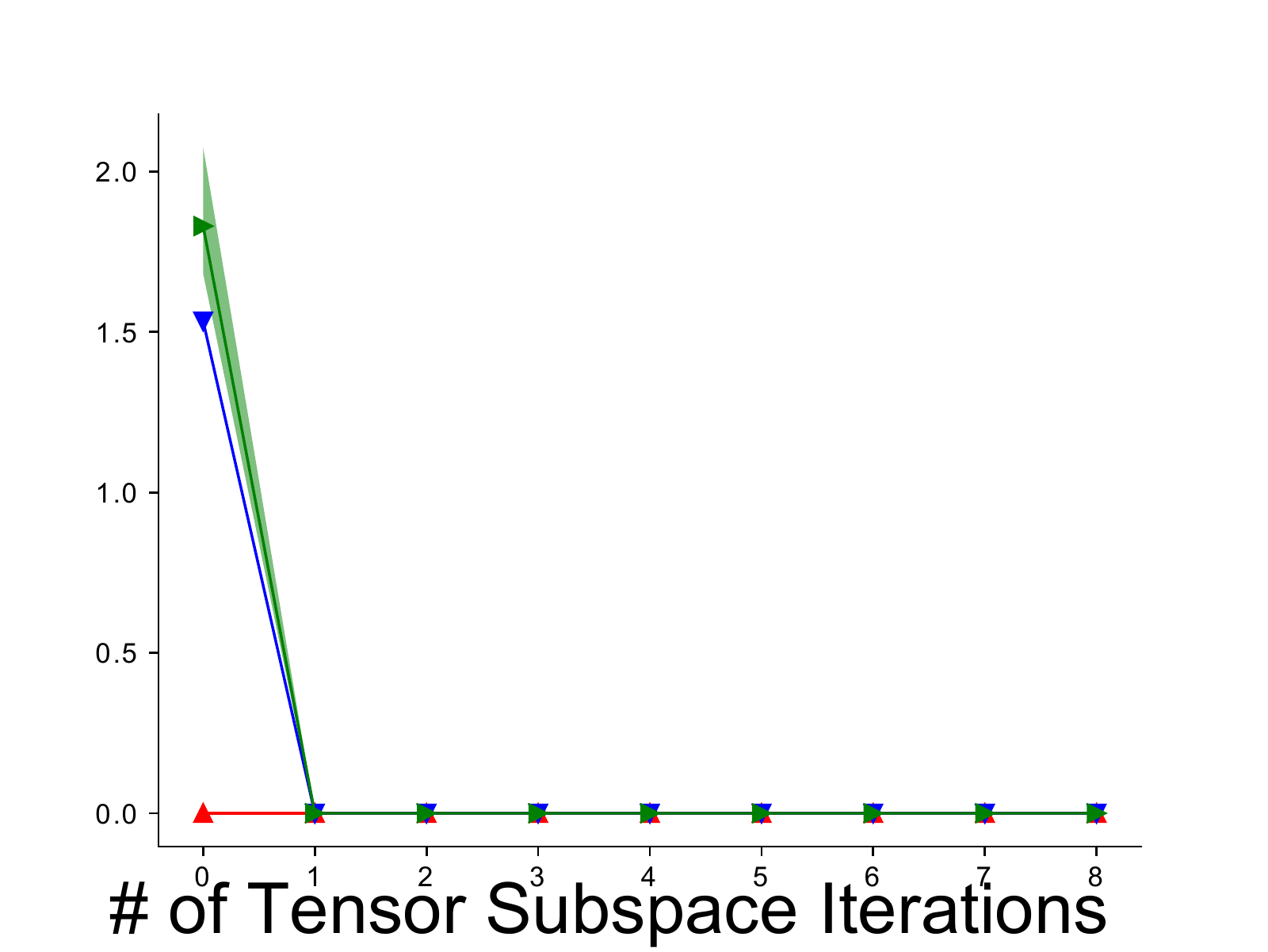}
			\caption{\scriptsize{$\tan{(\mmCols{C}{r}, \est{C})}$}}
			\label{fig:true_r_asym_tan3}
		\end{subfigure}
	\end{minipage}
	\caption{Convergence of our \ouralgoshort vs r-ALS~\cite{sharan2017orthogonalized} vs SPI~\cite{wang2017tensor} for asymmetric tensor when $r = R=10$, $d=500$ run 100 times.
	}
	\label{fig:true_r_asym}	
\end{figure}



\section{Conclusion}
Discovering latent variable models over large datasets can be cast as a tensor decomposition problem.  
Existing theory for tensor
decompositions guarantee results when the tensor is
symmetric.  
However, in practice, the tensors are noisy due to finite examples, and also inherently
asymmetric.  
Recovering top $r$ components of asymmetric tensors is often required for many learning scenarios.  
We present the
first algorithm for guaranteed recovery of tensor factors for an asymmetric noisy tensor.  
Our results extend to tensors with \emph{incoherent components}, where the orthogonality constraint is relaxed to tensors with nearly orthogonal components.  
 

\bibliographystyle{plain}\bibliography{\bibhome/supp_bib}
\newpage
\appendix
\begin{center}{\Large \textbf{Appendix: \mytitle}}\end{center}
\section{A Naive Initialization Procedure}\label{app:naive}

Based on the CP decomposition model in Equation~\eqref{eq:generative}, it is easy to see that the frontal slices shares the mode-A and mode-B singular vectors with the tensor  $\mytensor{T}$, and the $k^\tha$ frontal slice is $\mm{M}_{Ck} = \mm{A} \mm{\Lambda}_{Ck} \mm{B}^\top$ where $\mm{\Lambda}_{Ck} = 
	\begin{bmatrix}
		\lambda_1 c_{k1} &&\bm{0}\\
		&\ddots &\\
		\bm{0}&&\lambda_R c_{kR}
	\end{bmatrix}$. 
It is natural to consider naively implementing singular value decompositions on the frontal slices to obtain estimations of $\mymatrix{A}$ and $\mymatrix{B}$. 

\paragraph{Failure of Naive Initialization}

Consider the simpler scenario of finding a good initialization for a symmetric tensor $\mytensor{T}$ which permits the following CP decomposition
\begin{equation}
\mytensor{T} = \sum\limits_{i=1}^{R} \lambda_i \myvector{u}_i \otimes \myvector{u}_i \otimes \myvector{u}_i
\end{equation}

Specifically 
we have
\begin{align}
	\mytensor{T}(\mm{I}, \mm{I}, \bm{v}^C) = \mm{U\Lambda}^2 \mm{U}^\top
\end{align}
where $\mm{U} = [\bm{u}_1, \cdots, \bm{u}_R], \Lambda = \text{diag}(\lambda_1,\cdots, \lambda_R)$. However the first method gives us a matrix without any improvement on the diagonal decomposition, i.e. $  \mm{U} \mm{\Lambda}_U \mm{U}^\top$, where
\begin{align}
	\mm{\Lambda}_U = \text{diag}(\lambda_1 u_{k1} ,\cdots,\lambda_R u_{kR} )
\end{align}
For each eigenvalue of matrix $  \mm{U} \mm{\Lambda}_U \mm{U}^\top$, it contains not only the factor of a tensor singular value which we care about, but also some unknowns from the unitary matrix. This induces trouble when one wants to recover the subspace relative to only some leading singular values of the tensor if the rank $R$ is believed to be in a greater order of the dimension $d$. Although the analogous statement in matrix subspace iteration is true almost surely (with probability one),
in tensor subspace iteration we indeed need to do more work than simply taking a slice. It is highly likely that the unknown entries $u_{k1}, \cdots, u_{kR}$ permute the eigenvalues into an unfavorable sequence. Meanwhile, since $\mm{\Lambda}^2$ is ideally clean, we see success when we use the second method to recover the subspace relative to a few dominant singular values of a symmetric tensor.

They are all qualified in the sense that they own $\mm{A}$ as the left eigenspace exactly. However we can generalize this scheme to a greater extent. Frontal slicing is just a specific realization of multiplying the tensor on the third mode by a unit vector. Mode-$n$ product of a tensor with a vector would return the collection of inner products of each mode-$n$ fiber with the vector. The mode-3 product of tensor $\mytensor{T}$ with $\bm{e}_k$ will give the $k$th slice of $\mytensor{T}$.

\section{Unreliability of Symmetrization}\label{app:symmetrization}

In multi-view model, \cite{anandkumar2012spectral} introduced a method to symmetrize an asymmetric tensor. Here we change the notations and restate it below.

\begin{proposition}
	Let $\mytensor{T} = \sum_{i = 1}^{R} \lambda_{i} \myvector{u}_i \otimes \myvector{v}_i \otimes \myvector{w}_i$ have components $\mymatrix{U}, \mymatrix{V}, \mymatrix{W}$, then for some vectors $\myvector{a}$ and $\myvector{b}$ chosen independently, tensor \begin{equation} \mytensor{T}( \mytensor{T}(\myvector{b}, \identitymatrix, \identitymatrix)^\top  \mytensor{T}(\identitymatrix, \identitymatrix, \myvector{a})^{-1} ,  \mytensor{T}(\identitymatrix, \myvector{b}, \identitymatrix)^\top  (\mytensor{T}(\identitymatrix, \identitymatrix, \myvector{a})^\top)^{-1} , \identitymatrix) \end{equation} is symmetric.
	
\end{proposition}

\begin{proof}
	\begin{align}
	&\mytensor{T}(\identitymatrix, \identitymatrix, \myvector{a}) = \sum_{i = 1}^{R} \lambda_i(\myvector{w}_i^\top\myvector{a}) \myvector{u}_i \otimes \myvector{v}_i = \mymatrix{U} \diag(\lambda_i (\myvector{w}_i^\top\myvector{a})) \mymatrix{V}^\top \\
	&(\mytensor{T}(\identitymatrix, \identitymatrix, \myvector{a}))^{-1} = \mymatrix{V} \diag \big(\frac{1}{\lambda_i (\myvector{w}_i^\top\myvector{a})} \big) \mymatrix{U}^\top
	\end{align}
	Similarly,
	\begin{equation}
	\mytensor{T}(\myvector{b}, \identitymatrix, \identitymatrix) = \mymatrix{V}\diag(\lambda_{i}(\myvector{u}_i^\top\myvector{b}))\mymatrix{W}^\top, \quad 
	\mytensor{T}(\identitymatrix, \myvector{b}, \identitymatrix) = 
	\mymatrix{U}\diag(\lambda_{i}(\myvector{v}_i^\top\myvector{b}))\mymatrix{W}^\top
	\end{equation}
	Therefore,
	\begin{align}
	&\quad \mytensor{T}( \mytensor{T}(\myvector{b}, \identitymatrix, \identitymatrix)^\top  \mytensor{T}(\identitymatrix, \identitymatrix, \myvector{a})^{-1} ,  \mytensor{T}(\identitymatrix, \myvector{b}, \identitymatrix)^\top  (\mytensor{T}(\identitymatrix, \identitymatrix, \myvector{a})^\top)^{-1} , \identitymatrix  ) \\
	&= \mytensor{T}( \mymatrix{W} \diag \Big(\frac{ \myvector{u}_i^\top\myvector{b}}{  \myvector{w}_i^\top\myvector{a}} \Big) \mymatrix{U}^\top, \mymatrix{W} \diag \Big(\frac{ \myvector{v}_i^\top\myvector{b}}{  \myvector{w}_i^\top\myvector{a}} \Big) \mymatrix{V}^\top, \identitymatrix) \\
	&= \sum_{i = 1}^{R} \lambda_i \frac{ \myvector{u}_i^\top\myvector{b}}{  \myvector{w}_i^\top\myvector{a}} \frac{ \myvector{v}_i^\top\myvector{b}}{  \myvector{w}_i^\top\myvector{a}} \myvector{w}_i \otimes \myvector{w}_i \otimes \myvector{w}_i
	\end{align}
	shows the symmetry.
\end{proof}

However, in practice the condition number for $ \mytensor{T}(\identitymatrix, \identitymatrix, \myvector{a}) $ could be very large. So symmetrization using matrix inversion is not reliable since it is sensitive to noise. 

Indeed, we can analyze this assuming $ \bm{a}  $ is a fixed vector. Proposition~\ref{thm:daoshudierge} by Jiang et al.~\cite{jiang2006}
provides a good tool for our analysis.

\begin{proposition}\label{thm:daoshudierge} Let $\bm{M}_d = (m_{ij})_{1\le i,j \le d}$, where $m_{ij}$'s are independent standard Gaussian,  $\bm{X}_d = (x_{ij})_{1\le i,j \le d}$ be the matrix obtained from performing the Gram-Schmidt procedure on the columns of $\bm{M}_d$,  $ \{ n_d < d: d \ge 1 \}$ be a sequence of positive integers and
	\begin{align}
	\epsilon_d(n) \equiv \max\limits_{1 \le i \le d, 1\le j \le n} \big|\sqrt{d} x_{ij} -m_{ij}\big|, 
	\end{align}
	we then have
	\begin{enumerate}
		\item[(1)] the matrix $\bm{X}_d$ is Haar invariant on the orthonormal group $O(n)$;
		\item[(2)] $\epsilon_d(n_d) \rightarrow 0$ in probability, provided $n_d = o(d/\log d)$ as $n \rightarrow \infty$;
		\item[(3)] $\forall \alpha > 0$, we have that $~ \epsilon_d([d\alpha/\log d]) \rightarrow 2\sqrt{\alpha}$ in probability as $d \rightarrow \infty$.
	\end{enumerate}
\end{proposition}

This proposition states that for an orthonormal matrix generated by performing Gram-Schmidt procedure to standard normal matrix,
, the first $o(d/\log d)$ columns,
scaled by $\sqrt{d}$, asymptotically behave like a matrix
with independent standard Gaussian entries and this is the largest order
for the number of columns we can approximate simultaneously.

The condition number of matrix $ \mytensor{T}(\identitymatrix, \identitymatrix, \myvector{a}) $ is
\begin{equation}
\mathcal{K} (\mytensor{T}(\identitymatrix, \identitymatrix, \myvector{a})) = \frac{\max\limits_{1\le i\le R} | \lambda_{i} \myvector{w}_i^\top\myvector{a} |}{\min\limits_{1\le i\le R} |\lambda_{i}\myvector{w}_i^\top\myvector{a}|} ,
\end{equation}
which is nondecreasing as the rank of tensor $R$ increases. So we can indeed assume $R = o(d/\log d)$ and study the badness of condition number for such $\bm{W}$'s as worse cases.

\begin{remark}
	We treat $\bm{W}$ as the left $d \times R$ sub-block of some
	orthonormal matrix. Thus by assuming $R = o(d/\log d)$, $\sqrt{d}\bm{W}$ could be
	approximated by a matrix of i.i.d. $\mathcal{N}(0, 1)$ variables when $d$ is large, which is common in practice.
\end{remark}
Since condition number $\mathcal{K}$ is taking ratio, without loss of gernerality we can let $\| \bm{a} \| = 1$. Then,
\begin{equation}
\mathcal{K} (\mytensor{T}(\identitymatrix, \identitymatrix, \myvector{a})) = 
\frac{\max\limits_{1\le i\le R} | \lambda_{i} (\sqrt{d}\myvector{w}_i)^\top\myvector{a} |}{\min\limits_{1\le i\le R} |\lambda_{i}(\sqrt{d}\myvector{w}_i)^\top\myvector{a}|} .
\end{equation}
For $1 \le i \le R$, ~ $\lambda_{i}(\sqrt{d}\myvector{w}_i)^\top\myvector{a}$ are independent to each other and approximately has distribution $  \mathcal{N}(0, \lambda_i^2)$. So the condition number is approximately the ratio between maximum and minimum of absolute value of $\mathcal{N}(\bm{0}, \diag(\lambda_i^2))$. One can imgine if the tensor has one or more small singular values then it is highly likely for the condition number to be high.

\section{Procedure~\ref{algo:main} Noiseless Convergence Result}\label{sec:convergence}
 \subsection{Conditional Simultaneous Convergence}

 \begin{theorem}[Main Convergence]\label{thm:main}
 Using the initialization procedure~\ref{algo:init}, 
 	Denote the recovered tensor as $\mytensor{T}^* = \llbracket \est{\Lambda} ; \est{A}, \est{B}, \est{C} \rrbracket$ after $J = O(\log(C)/\log(|\frac{\lambda_{r}}{\lambda_{r+1}} |))$ iterations in initialization procedure~\ref{algo:init} and $K = O(\log(\log \frac{1}{\epsilon}))$ iterations in main procedure\ref{algo:main} applied on $\mathcal{T}$, $\forall \epsilon >0$. 
 	We have
 	\begin{align}
 		\| \mytensor{T}^* - \mytensor{T} \|_s \le \epsilon. 
 	\end{align}
 \end{theorem}
To prove the main convergence result, just combine all of the rest results together. 

 \begin{lemma}\label{thm:main_convergence}
 	Let $\mm{Q}_{\mmCols{A}{r}}^{(0)}, \mm{Q}_{\mmCols{B}{r}}^{(0)}, \mm{Q}_{\mmCols{C}{r}}^{(0)} , \forall r \in \{1,2,\cdots, R\} $, be $d\times r$ orthonormal initialization matrices for the specified subspace iteration.  
 	Then after $K$ iterations, we have 
 	\begin{align} 
 		t_{A_{(r)}}^{(K)}  \le \Big(\frac{\lambda_{r+1}}{\lambda_{(r)}}\Big)^{2^K-1}  \big( t_{A_{(r)}}^{(0)} t_{B_{(r)}}^{(0)} t_{C_{(r)}}^{(0)} \big)^{\frac{2^K}{3}} \Bigg[\frac{\big( t_{A_{(r)}}^{(0)} \big)^2}{ t_{B_{(r)}}^{(0)} t_{C_{(r)}}^{(0)
 		} }\Bigg]^{\frac{(-1)^K}{3}} , \quad \forall K \ge 1.
 	\end{align}
 	where $ t_{A_{(r)}}^{(k)} = \tan\Big(\mmCols{A}{r}, \mm{Q}_{\mmCols{A}{r}}^{(k)}\Big)$, $ t_{B_{(r)}}^{(k)} = \tan\Big(\mm{B}_{(r)}, \mm{Q}_{\mmCols{B}{r}}^{(k)}\Big)$, $t_{C_{(r)}}^{(k)} = \tan\Big(\mm{C}_{(r)}, \mm{Q}_{\mmCols{C}{r}}^{(k)}\Big)$, $\forall k \ge 0$.
 	Similarly for $\mm{B}_{(r)}$ and $\mm{C}_{(r)}$.
 \end{lemma}
 The proof is in Appendix~\ref{app:main_convergence}.

\begin{remark}
	Given that the initialization matrices $\mm{Q}_{\mmCols{A}{r}}^{(0)}, \mm{Q}_{\mmCols{B}{r}}^{(0)}, \mm{Q}_{\mmCols{C}{r}}^{(0)}$ satisfy the $r$-sufficient initialization condition, the angles between approximate subspaces and true spaces would decrease with a quadratic rate. Therefore, only $K=O(\log(\log \frac{1}{\epsilon}))$ number of iterations is needed to achieve $\tan(\mmCols{A}{r}, \mm{Q}_{\mmCols{A}{r}}^{(K)}) \le \epsilon$.
\end{remark}

 The following result shows that if we have the angle of subspaces small enough, column vectors of the approximate matrix converges simultaneously to the true vectors of true tensor component at the same position.
 
 \begin{lemma}[Simultaneous Convergence]\label{thm:simul_convergence}
 	For any $r \in \{1,2,\cdots,R\}$, if \begin{equation}\tan(\mmCols{A}{r}, \mm{Q}_{\mmCols{A}{r}}) \le \epsilon\end{equation} for some $d \times r$ matrix $\mm{Q}_{\mmCols{A}{r}} = [\bm{q}_1, \cdots, \bm{q}_r]$, then
 	\begin{align} 
 		\| \bm{q}_i - \bm{a}_i \|^2 \le 2\epsilon, \quad \forall 1\le i \le r . 
 	\end{align}
 	Similarly for $\mm{B}_{(r)}$ and $\mm{C}_{(r)}$.
 \end{lemma}
 The proof is in Appendix~\ref{app:simul_convergence}.
 

\subsection{Proof for Lemma~\ref{thm:main_convergence}}\label{app:main_convergence}
\begin{proof}
		We only prove the result for the order of $A$. The proofs for the other two orders are the same.
		
        For rank-$R$ tensor $\mathcal{T} = \llbracket \mm{\Lambda} ; \mm{A}, \mm{B}, \mm{C} \rrbracket \equiv \sum_{i = 1}^{R} \lambda_i  \bm{a}_i \otimes \bm{b}_i \otimes \bm{c}_i $, its mode-1 matricization $\mathcal{T}_{(1)} = \mm{A}\mm{\Lambda} (\mm{C}\odot \mm{B})^\top $. So in each iteration, 
		\begin{align}
		\mm{Q}_{\mmCols{A}{r}}^{(k+1)}\mm{R}_{\mmCols{A}{r}}^{(k+1)}  & =  \mathcal{T}_{(1)}(\bm{Q}_{\mmCols{C}{r}}^{(k)} \odot \bm{Q}_{\mmCols{B}{r}}^{(k)} ) = \bm{A\Lambda} (\bm{C}\odot \bm{B})^\top (\bm{Q}_{\mmCols{C}{r}}^{(k)} \odot \bm{Q}_{\mmCols{B}{r}}^{(k)} ) \\
		\text{and by property} & \text{ of Hadamard product and Khatri-Rao product ~\cite{liu2008hadamard,kolda2009tensor}},\nonumber\\
		& = \bm{A\Lambda} (\bm{C}^\top \bm{Q}_{\mmCols{C}{r}}^{(k)}) \ast (\bm{B}^\top \bm{Q}_{\mmCols{B}{r}}^{(k)}) 
		\end{align}
		
		 We can expand matrices $\bm{A}, \bm{B}, \bm{C}$ to be a basis for $\mathbb{R}^d$, and we can for example for $\mmCols{A}{r}$, let $\mmCols{A}{r}^c$ be the matrix consisted of the rest $(d-r)$ columns in the expanded matrix. Now the column space of $\mmCols{A}{r}^c$ is just the complement space of column space of $\mmCols{A}{r}$ in $\mathbb{R}^d$. And $\big[\mmCols{A}{r} ~ \mmCols{A}{r}^c\big]$ is a $d\times d$ orthonormal matrix.
		 
		 With that notation, we have for $0\le k \le K,$
		 \begin{align*}
		 \mmCols{A}{r}^\top \bm{Q}_{\mmCols{A}{r}}^{(k+1)} \bm{R}_{\mmCols{A}{r}}^{(k+1)}  &= \Big[ \bm{I}_r ~~ \bm{0}_{r\times (R-r)}  \Big] \bm{\Lambda} \big(\bm{C}^\top \bm{Q}_{\mmCols{C}{r}}^{(k)}\big) \ast \big(\bm{B}^\top \bm{Q}_{\mmCols{B}{r}}^{(k)}\big) \\
		 \mmCols{A}{r}^{c\top} \bm{Q}_{\mmCols{A}{r}}^{(k+1)} \bm{R}_{\mmCols{A}{r}}^{(k+1)}  &= 
		 \begin{bmatrix}
		 \bm{0}_{(R-r) \times r} & \bm{I}_{(R-r) \times (R-r)} \\
		 \bm{0}_{(d-R) \times r} & \bm{0}_{(d-R) \times (R-r)}
		 \end{bmatrix}
		  \bm{\Lambda} \big(\bm{C}^\top \bm{Q}_{\mmCols{C}{r}}^{(k)}\big) \ast \big(\bm{B}^\top \bm{Q}_{\mmCols{B}{r}}^{(k)}\big).
		 \end{align*}
		 
		 Now fix $k$ and focus on a single iteratoin step, 
		 \begin{align*}
		 t_{A_r}^{(k+1)} &= \tan(\mmCols{A}{r}, \bm{Q}_{\mmCols{A}{r}}^{(k+1)}) = \frac{\sin(\mmCols{A}{r}, \bm{Q}_{\mmCols{A}{r}}^{(k+1)})}{\cos(\mmCols{A}{r}, \bm{Q}_{\mmCols{A}{r}}^{(k+1)})} = \frac{\sigma_{\text{max}}(\mmCols{A}{r}^{c\top}\bm{Q}_{\mmCols{A}{r}}^{(k+1)})}{\sigma_{\text{min}}(\mmCols{A}{r}^\top \bm{Q}_{\mmCols{A}{r}}^{(k+1)})} \\
		 &= \Big\| \mmCols{A}{r}^{c\top}\bm{Q}_{\mmCols{A}{r}}^{(k+1)}  \Big\|_s   \Big\| \Big(  \mmCols{A}{r}^\top \bm{Q}_{\mmCols{A}{r}}^{(k+1)} \Big)^{-1} \Big\|_s \\
		 &= \Big\|	\mmCols{A}{r}^{c\top}\bm{Q}_{\mmCols{A}{r}}^{(k+1)}  \Big(  \mmCols{A}{r}^\top \bm{Q}_{\mmCols{A}{r}}^{(k+1)} \Big)^{-1} 	\Big\|_s \\
		 &= \Big\|	\mmCols{A}{r}^{c\top}\bm{Q}_{\mmCols{A}{r}}^{(k+1)} \bm{R}_{\mmCols{A}{r}}^{(k+1)}  \Big(  \mmCols{A}{r}^\top \bm{Q}_{\mmCols{A}{r}}^{(k+1)} \bm{R}_{\mmCols{A}{r}}^{(k+1)} \Big)^{-1} 	\Big\|_s \\
		 &\le \frac{  \sigma_{\text{max}} \Big( \mmCols{A}{r}^{c\top}\bm{Q}_{\mmCols{A}{r}}^{(k+1)} \bm{R}_{\mmCols{A}{r}}^{(k+1)}   \Big)  }{\sigma_{\text{min}} \Big( \mmCols{A}{r}^\top \bm{Q}_{\mmCols{A}{r}}^{(k+1)} \bm{R}_{\mmCols{A}{r}}^{(k+1)}  \Big)}\\
		 &\le \frac{\lambda_{r+1} \sigma_{\text{max}} \Big[ \Big(\bm{C}_{(r)}^{c\top} \bm{Q}_{\mmCols{C}{r}}^{(k)} \Big) \ast \Big(\mmCols{B}{r}^{c\top} \bm{Q}_{\mmCols{B}{r}}^{(k)} \Big) \Big]  }{\lambda_{r} \sigma_{\text{min}} \Big[ \Big(\mmCols{C}{r}^{c\top} \bm{Q}_{\mmCols{C}{r}}^{(k)} \Big) \ast \Big(\mmCols{B}{r}^{c\top} \bm{Q}_{\mmCols{B}{r}}^{(k)} \Big) \Big] } \\
		 &\text{For Hadamard product}, ~ \sigma_{\text{max}}(\bm{M}_1 \ast \bm{M}_2) \le \sigma_{\text{max}}(\bm{M}_1) \sigma_{\text{max}}(\bm{M}_2) \\
		 &\text{and} ~ \sigma_{\text{min}}(\bm{M}_1 \ast \bm{M}_2) \ge \sigma_{\text{min}}(\bm{M}_1) \sigma_{\text{min}}(\bm{M}_2) \text{see ~\cite{liu2008hadamard}} \\
		 & \le \frac{\lambda_{r+1}}{\lambda_{r}}  \frac{\sigma_{\text{max}}\Big(\mmCols{C}{r}^{c\top} \bm{Q}_{\mmCols{C}{r}}^{(k)} \Big)}{\sigma_{\text{min}}\Big(\mmCols{C}{r}^{c\top} \bm{Q}_{\mmCols{C}{r}}^{(k)} \Big)} \frac{\sigma_{\text{max}}\Big(\mmCols{B}{r}^{c\top} \bm{Q}_{\mmCols{B}{r}}^{(k)} \Big)}{\sigma_{\text{min}}\Big(\mmCols{B}{r}^{c\top} \bm{Q}_{\mmCols{B}{r}}^{(k)} \Big)} \\
		 & = \frac{\lambda_{r+1}}{\lambda_{r}} \cdot \tan\Big(\mmCols{B}{r}, \bm{Q}_{\mmCols{B}{r}}^{(k)}\Big) \cdot \tan\Big(\mmCols{C}{r}, \bm{Q}_{\mmCols{C}{r}}^{(k)}\Big)
		 \end{align*}
		 
		 Therefore we get $\forall 0\le k \le K$,
		 \[ t_{A_r}^{(k+1)} \le \frac{\lambda_{r+1}}{\lambda_{r}} t_{B_r}^{(k)}  t_{C_r}^{(k)} .  \]
		 
		 And similarly,
		 \[ t_{B_r}^{(k+1)} \le \frac{\lambda_{r+1}}{\lambda_{r}} t_{A_r}^{(k)}  t_{C_r}^{(k)} , \]
		 \[ t_{C_r}^{(k+1)} \le \frac{\lambda_{r+1}}{\lambda_{r}} t_{A_r}^{(k)}  t_{B_r}^{(k)} .  \]
		 
		 Sequentially,
		 \begin{align*}
		 t_{A_r}^{(K+1)} ~ & ~ \le ~ \frac{\lambda_{r+1}}{\lambda_{r}} t_{B_r}^{(K)}  t_{C_r}^{(K)} ~\le~  \Big( \frac{\lambda_{r+1}}{\lambda_{r}}  \Big)^3  (t_{A_r}^{(K-1)} )^2 t_{B_r}^{(K-1)}  t_{C_r}^{(K-1)} \\
		 & ~\le~ \cdots ~ \le ~ \Big( \frac{\lambda_{r+1}}{\lambda_{r}}  \Big)^{1+2m} \Big(\prod_{i=1}^{m} (t_{A_r}^{(K-i)} )^2\Big) t_{B_r}^{(K-m)}  t_{C_r}^{(K-m)} \quad \forall m = 1,2,\dots, K
		 \end{align*}
		 
		 Easy to see that all historical tangents of principal angle in approximation for $\mmCols{A}{r}$ appear in the upper bound for the tangent-measured approximation distance after a new iteration. So in order to solve for the explicit upper bounds, we can assume the form of the upper bounds has a recursive formula for each exponents. Specifically, assume for some sequences $u_K$, $a_K$, $b_K$, we can conclude
		 
		\[ t_{A_r}^{K+1} ~ \le ~ \Big( \frac{\lambda_{r+1}}{\lambda_{r}}  \Big)^{u_{K+1}} \big(t_{A_r}^{(0)} \big)^{a_{K+1}} \big(t_{B_r}^{(0)} t_{A_r}^{(0)} \big)^{b_{K+1}} \]
		
		On the other hand, for fixed $K \ge 1$,
		\begin{align*}
		t_{A_r}^{K+1} ~ & \le ~  \Big( \frac{\lambda_{r+1}}{\lambda_{r}}  \Big)^{1+2K} \Big(\prod_{i=1}^{K} (t_{A_r}^{(K-i)} )^2\Big) t_{B_r}^{(0)}  t_{C_r}^{(0)}  \\
		& ~ \le \Big( \frac{\lambda_{r+1}}{\lambda_{r}}  \Big)^{1+2K} ~ \prod_{i=1}^{K} \Big[ \Big( \frac{\lambda_{r+1}}{\lambda_{r}}  \Big)^{u_{K-i}} \big(t_{A_r}^{(0)} \big)^{a_{K-i}} \big(t_{B_r}^{(0)} t_{A_r}^{(0)} \big)^{b_{K-i}}\Big]^2~ \cdot t_{B_r}^{(0)}  t_{C_r}^{(0)}  \\
		& ~ = \Big( \frac{\lambda_{r+1}}{\lambda_{r}}  \Big)^{1+2K+2\sum_{i=1}^{K} u_{K-i} }   \big(t_{A_r}^{(0)} \big)^{2\sum_{i=1}^{K} a_{K-i} }  \big(t_{B_r}^{(0)} t_{A_r}^{(0)} \big)^{1+ 2\sum_{i=1}^{K}  b_{K-i}}
		\end{align*}
		
		Now we have gained the recursive formulas for sequence on exponents in the upper bound
		\begin{align*}
		u_{K+1} &= 1 + 2K + 2\sum_{i=1}^{K} u_{K-i}  \\
		a_{K+1} &= 2\sum_{i=1}^{K} a_{K-i} \\
		b_{K+1} &= 1+ 2\sum_{i=1}^{K}  b_{K-i} .
		\end{align*}
		
		The formula system works on when $K \ge 1$, so we can check the upper bounds for several initial iterations.
		
		For $K = 0$, 
		\[ t_{A_r}^{(1) } \le \frac{\lambda_{r+1}}{\lambda_{r}} t_{B_r}^{(0)} t_{C_r}^{(0)}   \]
		
		For $K = 1$,
		\[ t_{A_r}^{(2) } \le \Big( \frac{\lambda_{r+1}}{\lambda_{r}} \Big)^3 \big( t_{A_r}^{(0)}\big)^2 t_{B_r}^{(0)} t_{C_r}^{(0)}   \]
		
		For $K = 2$,
		\[ t_{A_r}^{(3) } \le \Big( \frac{\lambda_{r+1}}{\lambda_{r}} \Big)^7 \big( t_{A_r}^{(0)}\big)^2 \big( t_{B_r}^{(0)} t_{C_r}^{(0)} \big)^3   \]
		
		We have 
		\[ u_0 = 0, u_1 = 1, u_2 = 3, u_3 = 7, u_4 = 15, \dots  \]
		\[ a_0 = 1, a_1 = 0, a_2 = 2, a_ 3 = 2, a_ 4 = 6, \dots \]
		\[ b_0 = 0, b_1 = 1, b_2 = 1, b_3 = 3, b_4 = 5, \dots  \]
		
		One can solve and check the general formula for these sequences
		\[ u_K = 2^{K} -1, \quad a_K = \frac{2}{3} (2^{K-1} + (-1)^{K} ), \quad b_K = \frac{1}{3} (2^K + (-1)^{K-1} ), \quad \forall K \ge 1 .  \]
		
		In conclusion,
		\[ t_{A_r}^{(K)}  \le \Big(\frac{\lambda_{r+1}}{\lambda_r}\Big)^{2^K-1}  \big( t_{A_r}^{(0)} t_{B_r}^{(0)} t_{C_r}^{(0)} \big)^{\frac{2^K}{3}} \Bigg[\frac{\big( t_{A_r}^{(0)} \big)^2}{ t_{B_r}^{(0)} t_{C_r}^{(0)
		  } }\Bigg]^{\frac{(-1)^K}{3}} , \quad \forall K \ge 1.\]
		
		The proofs of upper bounds for $\mmCols{B}{r}$ and $\mmCols{C}{r}$ are the same.
		 
\end{proof}

\subsection{Proof for Lemma~\ref{thm:simul_convergence}}\label{app:simul_convergence}
	\begin{proof}
	First, we denote $\mmCols{Q}{i} := [\bm{q}_1, \cdots, \bm{q}_i]$ only in this proof. Then
	\begin{align*}
	\tan(\mmCols{A}{r-1}, \mmCols{Q}{r-1}) &= \frac{\sqrt{1- \sigma_{\text{min}}^2(\mmCols{A}{r-1}^\top \mmCols{Q}{r-1})}}{\sigma_{\text{min}}(\mmCols{A}{r-1}^\top \mmCols{Q}{r-1})} \\
	& = \sqrt{\frac{1}{\sigma_{\text{min}}^2(\mmCols{A}{r-1}^\top \mmCols{Q}{r-1})}-1} \\
	& \quad \text{by Cauchy interlacing theorem}\\
	& \le \sqrt{\frac{1}{\sigma_{\text{min}}^2(\bm{A}_{(r)}^\top \bm{Q}_{(r)})}-1}  \\
	& = \tan(\bm{A}_{(r)}, \bm{Q}_{(r)})
	\end{align*}
	
	Inductively, $\forall 1 \le i \le r, ~ \tan(\mmCols{A}{i}, \mmCols{Q}{i}) \le \epsilon$. Then $\forall 2 \le i \le r$,
	\begin{align*}
	\cos^2(\mmCols{A}{i}, \mmCols{Q}{i}) & = \min\limits_{\bm{y} \in \mathbb{R}^i} \frac{\| \mmCols{Q}{i}^\top \mmCols{A}{i} \bm{y} \|^2 }{\| \mmCols{A}{i} \bm{y} \|^2} \\
	& \le \| \mmCols{Q}{i}^\top \bm{a}_i \|^2  \quad \text{as letting} ~ \bm{y} \text{ to be } [0, \cdots, 0, 1]^\top \\
	& = \| \mmCols{Q}{i-1}^\top \bm{a}_i \|^2 + (\bm{q}_i^\top \bm{a}_i )^2 \\
	&  \quad \text{since $\bm{a}_i \in \mathscr{C}(\mmCols{A}{i-1})^\perp $ , the complement} \\
	& \quad \text{space of column space of $\mmCols{A}{i-1}$   }\\
	& \le \sin^2(\mmCols{A}{i-1}, \mmCols{Q}{i-1}) + (\bm{q}_i^\top \bm{a}_i )^2 
	\end{align*}
	\begin{align*}
	(\bm{q}_i^\top \bm{a}_i )^2 & \ge \frac{1}{1+\tan^2(\mmCols{A}{i}, \mmCols{Q}{i} )} - \frac{\tan^2(\mmCols{A}{i-1}, \mmCols{Q}{i-1})}{1+ \tan^2(\mmCols{A}{i-1}, \mmCols{Q}{i-1})} \\
	& \ge \frac{1}{1 + \epsilon^2} -1 + \frac{1}{1+ \epsilon^2} = 1- \frac{2\epsilon^2}{1+\epsilon^2} \ge 1- 2 \epsilon^2 . 
	\end{align*}
	
	For $i = 1$, 
	\[ \cos^2(\mmCols{A}{1}, \mmCols{Q}{1}) = (\bm{q}_1^\top \bm{a}_1 )^2 = \frac{1}{1+\tan^2(\mmCols{A}{1}, \mmCols{Q}{1} )} \ge \frac{1}{1 + \epsilon^2} \ge 1- 2 \epsilon^2 . \]
	
	To conclude, $ \| \bm{q}_i - \bm{a}_i \|^2 = 2 - 2 \bm{q}_i^\top \bm{a}_i \le 2\epsilon , \quad \forall 1 \le i \le r $. And the proofs for $\mmCols{B}{r}$ and $\mmCols{C}{r}$ are the same.

\end{proof}

\subsection{Lemma~\ref{thm:matrixsubspaceiteration} and Proof}\label{app:matrixsubspaceiteration}

\begin{lemma}\label{thm:matrixsubspaceiteration}
    Let $\mmCols{U}{p} , \mmCols{V}{p} \in \mathbb{R}^{d \times p}$ respectively be the orthonormal complex matrix whose column space is the left and right invariant subspace corresponding to the dominant $p$ eigenvalues of $\mm{M}\in \R^{d\times d}$. Assume for fixed initialization $\bm{Q}^{(0)}$, $\mmCols{V}{p}^\top\bm{Q}^{(0)}$ has full rank. 
	Then after $\forall k \ge 1$ steps (independent of $\epsilon$) of matrix subspace iteration $\mymatrix{Q}^{(k)} \mymatrix{R}^{(k)} \leftarrow \qr{\mymatrix{MQ}^{(k-1)}}$,  we obtain 
	$\tan(\mmCols{U}{p}, \bm{Q}^{(k)}) \le C \cdot  \Big| \frac{\sigma_{p+1}(\mm{M})}{\sigma_{p}(\mm{M})}  \Big| ^k$ 
	for a finite constant $C$, where $\sigma_p(\cdot)$ denotes the $p^\tha$ singular value.
\end{lemma}

\begin{proof}
	Since $\bm{A}$ is orthogonal in the way $\bm{AA^*} = \bm{A^*A}$, $\bm{A}$ is a normal matrix. So its Schur decomposition and eigendecomposition coincides to $\bm{A} = \bm{PDP^*}$. Here $\bm{PP^*} = \bm{P^*P} = \bm{I}$. $\bm{D}$ is a diagonal matrix with all eigenvalues of $\bm{A}$ on diagonal and without loss of generality we can permutate them to be in a decreasing order, i.e. $\bm{D} = diag(\lambda_1, \cdots, \lambda_p, \lambda_{p+1}, \cdots, \lambda_d)$. We can furthermore denote $\bm{D} = 
	\begin{bmatrix}
	\bm{D}_1 & \bm{0} \\
	\bm{0} & \bm{D}_2 
	\end{bmatrix}
	$, where $\bm{D}_1$ contains eigenvalues up to $\lambda_p$ and  $\bm{D}_2$ contains eigenvalues $\lambda_{p+1}$ to $\lambda_d$.
	
	Inspired by ~\cite{arbenz2012lecture}, without making any restriction to the matrix to initialize the algorithm, we can assume the iterations take place in the space of $\{ \bm{PQ}  \}$ without loss of generality because $\bm{P}$ is invertible. Then we notice that for the iteration formula, it becomes
	\[ \bm{PQ^{(k)}R^{(k)} := APQ^{(k-1)} } \]
	\[ \bm{Q^{(k)}R^{(k)} := P^*APQ^{(k-1)} } \]
	\[ \bm{Q^{(k)}R^{(k)} := DQ^{(k-1)} }  \]
	
	So analytically, the convergence for an arbitrary matrix is the same to the convergence for the diagonal matrix formed from the eigenvalues of that matrix. And the left invariant eigenvector subspace for $\bm{D}$ is nothing but $\mmCols{E}{p} = [\bm{e_1, \cdots, e_p}] $. Imgine now $\bm{Q}^{(0)}$ is prepared to run the algorithm for $\bm{D}$, next we will show the subspace of $\bm{Q}^{(k)}$'s will converge to column space of $\mmCols{E}{p}$.
	
	First, partition $\bm{Q}^{(k)}$ to $
	\begin{bmatrix}
		\bm{Q}_{1}^{(k)} \\
		\bm{Q}_{2}^{(k)} 
	\end{bmatrix}$
	such that $\bm{Q}_{1}^{(k)} \in \mathbb{C}^{p \times p}$. $\bm{D}_1 \in \mathbb{C}^{p \times p}$ is invertible because of the eigenvalue gap. By the assumption that $\bm{V}_p^*\bm{Q}$ has full rank, here we have $\bm{Q}_{1}^{(0)} $ has full rank and thus invertible.  $\bm{Q}_{1}^{(k)} $ is therefore invertible.
	
	Notice that inductively,
	\[ \bm{Q^{(k)}R^{(k)} = DQ^{(k-1)} }  \]
	\[ \bm{Q^{(k)}R^{(k)}R^{(k-1)} = DQ^{(k-1)}R^{(k-1)} = D^2 Q^{(k-2)} }  \]
	\[ \bm{Q^{(k)}R^{(k)}R^{(k-1)} \cdots R^{(1)} = D^k Q^{(0)} = Q^{(k)}R}  \]
	for some upper-triangular matrix $\bm{R}$. Then
	\[ \bm{Q}^{(k)} \bm{R}= \bm{D}^k \bm{Q}^{(0)} = 
	\begin{bmatrix}
	\bm{D}_1^k \bm{Q}_1^{(0)} \\
	\bm{D}_2^k \bm{Q}_2^{(0)} 
	\end{bmatrix}  . \]
	\[ \bm{Q}^{(k)} = 
	\begin{bmatrix}
	\bm{D}_1^k \bm{Q}_1^{(0)} \bm{R} ^{-1}  \\
	\bm{D}_2^k \bm{Q}_2^{(0)}  \bm{R} ^{-1} 
	\end{bmatrix}  \]
	
	To study tangent, first look at
	\begin{align*}
	\sin(\mmCols{E}{p}, \bm{Q}^{(k)}) &= \| 
	\begin{bmatrix}
	\bm{0} & \bm{I}_{d-p} 
	\end{bmatrix}^\top \bm{Q}^{(k)}  \| _s = \| \bm{D}_2^k \bm{Q}_2^{(0)}  \bm{R} ^{-1}  \|_s \\
	& = \frac{\| \bm{D}_2^k \bm{Q}_2^{(0)}  \bm{R} ^{-1} (\bm{D}_1^k \bm{Q}_1^{(0)} \bm{R} ^{-1} )^{-1} \|_s}{ \sqrt{1+  \| \bm{D}_2^k \bm{Q}_2^{(0)}  \bm{R} ^{-1} (\bm{D}_1^k \bm{Q}_1^{(0)} \bm{R} ^{-1} )^{-1} \|_s^2  }} \\
	& \quad \text{Denote} ~  \bm{M}^{(k)} := \bm{D}_2^k \bm{Q}_2^{(0)} \big(\bm{Q}_1^{(0)} \big)^{-1}  \bm{D}_1^{-k} \\
	& = \frac{\| \bm{M}^{(k)}  \|_s }{\sqrt{1+\| \bm{M}^{(k)}  \|_s^2 }} .
	\end{align*}
	
	Correspondingly,
	\[ \cos(\mmCols{E}{p}, \bm{Q}^{(k)}) =\frac{1}{\sqrt{1+\| \bm{M}^{(k)}  \|_s^2 }}   \]
	
	Since spectral radius $\rho(\bm{D}_1^{-1}) = | \lambda_{p}|^{-1}, \rho(\bm{D}_2) = |\lambda_{p+1}|$, for any $\epsilon > 0 $, there exists a norm $\| \cdot \|_{(1)}$ such that $\| \bm{D}_1^{-1} \|_{(1) }\le  |\lambda_{p}|^{-1} + \epsilon$, and another norm $\| \cdot \|_{(2)}$ such that $\| \bm{D}_2 \|_{(2)} \le  |\lambda_{p+1}| + \epsilon$. By equivalence of norms, There exists constants $C_1, C_2 < \infty$ such that $\| \bm{M}\|_s \le C_1 \| \bm{M} \|_{(1)}$ and $\| \bm{M} \|_s \le C_2 \| \bm{M} \|_{(2)}$ for any matrix $\bm{M} $. 
	
	As a consequence, 
	\begin{align*}
	\tan(\mmCols{E}{p}, \bm{Q}^{(k)}) & = \| \bm{M}^{(k)}  \|_s  \le \| \bm{D}_1^k\|_s  \| \bm{M}^{(0)}  \|_s \| \bm{D}_2^{-k} \|_s \\
	& \le C_1C_2  \| \bm{D}_1^k\|_{(1)}  \| \bm{M}^{(0)}  \|_s \| \bm{D}_2^{-k} \|_{(2)} \\
	& \le C_1C_2 \tan(\mmCols{E}{p}, \bm{Q}^{(0)}) \| \bm{D}_1\|_{(1)}^k   \| \bm{D}_2^{-1} \|_{(2)}^k \\
	& \le C \Big( \big(| \lambda_{p+1} | +\epsilon \big) \big(\frac{1}{|\lambda_{p}|}+\epsilon\big)  \Big)^k
	\end{align*}
	for some constant $C$ after an initialization is chosen and fixed.
	
	Let $\epsilon_0$ be  $ (| \lambda_{p+1} | + \frac{1}{|\lambda_{p}|} +\epsilon) \epsilon$, then equivalently,
	\[ \tan(\mmCols{E}{p}, \bm{Q}^{(k)}) \le C \Big( \Big| \frac{\lambda_{p+1}}{\lambda_{p}}  \Big| + \epsilon_0 \Big)^k, \quad \forall \epsilon_0>0 .  \]
	
	This shows the convergence of subspace iteration algorithm on recovering the left eigenspace of a matrix in complex diagonal orthonormal matrix space with a specific eigenvalue gap. By the analytical equivalence dicussed before, we have identical convergence on recovering the left eigenspace of an arbitrary orthonormal matrix. In this way, equivalently, if $\bm{Q}^{(0)}$ is for this algorithm on $\bm{A}$,
	\[ \tan(\mmCols{U}{p}, \bm{Q}^{(k)}) \le C \Big( \Big| \frac{\lambda_{p+1}}{\lambda_{p}}  \Big| + \epsilon_0 \Big)^k, \quad \forall \epsilon_0>0 .  \]
	
	By taking infimum on $\epsilon_0$, it becomes
	\[ \tan(\mmCols{U}{p}, \bm{Q}^{(k)}) \le C \cdot  \Big| \frac{\lambda_{p+1}}{\lambda_{p}}  \Big| ^k \]
	
\end{proof}

\begin{remark}
	The condition that $\mmCols{V}{p}^\top\bm{Q}$ has full rank assumed in lemma \ref{thm:matrixsubspaceiteration} is satisfied almost surely (with probability 1).
\end{remark}
\begin{proof}
	 As a common procedure, to generate a random ($d \times r$)-sized orthonormal matrix, one could first generate a matrix of $r$ columns sampled i.i.d. from $d$-dimensional standard normal distribution, and then perform Gram-Schmidt algorithm on columns. Consider Gram-Schmidt algorithm as a mapping. Then under such mapping, the pre-image of a orthonormal matrix $[ \bm{q}_1, \bm{q}_2, \cdots, \bm{q}_r  ]$ is $[ s_{1}\bm{q}_1, s_{21}\bm{q}_1+s_{22}\bm{q}_2, \cdots, s_{r1}\bm{q}_1+\cdots+s_{rr}\bm{q}_r  ]$, for some constants $s_1, s_{21}, \cdots, s_{rr} \in \mathbb{R}$.  The columns of the pre-image (sampled from i.i.d. $\mathcal{N}(\bm{0}, \identitymatrix_d)$) belong to a subspace in $\mathbb{R}^d$. \\ The condition that $\mmCols{V}{p}^\top\bm{Q}$ has full rank is equivalent to the condition that there exists at least one column of $\bm{Q}$ that is in the complement of column space of $\mmCols{V}{p}$ in $\mathbb{R}^d$. So as long as the column space of $\mmCols{V}{p}$ is not the whole $\mathbb{R}^d$, in order to make $\mmCols{V}{p}^\top\bm{Q}$ not a full-rank matrix, at least one column of the random normal matrix has to take place in a proper subspace in $\mathbb{R}^d$. The multi-variate normal distribution is also a finite measure on $\mathbb{R}^d$. Therefore the measure of that proper subspace (i.e. the probability that we fail to have a full-rank $\mmCols{V}{p}^\top\bm{Q}$) is zero.
\end{proof}

\section{Lemma~\ref{thm:initSymmetric} and Proof}\label{app:oneandonlyonelemma}
\begin{lemma}\label{thm:initSymmetric}
	Mode-3 product of symmetric tensor $\mytensor{T}$ with vector $\mv{v}^C$ has the form
	$\mytensor{T}(\identitymatrix, \identitymatrix, \mv{v}^C)  = \mm{U}\mm{\Lambda}^2 \mm{U}^\top $ 	
where $\mm{\Lambda}  =\diag((\lambda_m)_{1\le m\le R}), \mm{U}=[\mv{u}_1,\ldots,\mv{u}_R].$
\end{lemma}

\begin{proof}We will prove a more general case for asymmetric tensor.
	$\mathcal{T}(\bm{I}, \bm{I}, \bm{v}^C)$ is a matrix. The $(i,j)$th entry of the matrix would be
	\begin{align*}
	[\mathcal{T}(\bm{I}, \bm{I}, \bm{v}^C)]_{ij} &= \sum_{k=1}^{d} \Big( \sum_{l=1}^{d} \sum_{m_1=1}^{R} \lambda_{m_1} a _{lm_1}b_{lm_1} c_{km_1}  \Big) \cdot \Big(\sum_{m_2=1}^{R} \lambda_{m_2} a_{im_2} b_{jm_2} c_{km_2} \Big)\\
	&=\sum_{m_1, m_2 = 1}^{R} \sum_{l=1}^{d} \lambda_{m_1}\lambda_{m_2}  a_{lm_1} a_{im_2}b_{lm_1}b_{jm_2} \sum_{k=1}^{d} c_{km_1}c_{km_2} \\
	& \quad \text{Because} \sum_{k=1}^{d} c_{km_1}c_{km_2} = 
	\begin{cases}
	=0& \text{if} ~ m_1 \neq m_2 \\
	=1& \text{if} ~ m_1 = m_2
	\end{cases}. \\
	& = \sum_{m=1}^{R} \Big(\lambda_m^2 \sum_{l=1}^{d} a_{lm}b_{lm} \Big) a_{im} b_{jm} \\
	& = \sum_{m=1}^{R} (\lambda_m^2 \bm{a}_m^\top \bm{b}_m ) a_{im} b_{jm} .
	\end{align*}
	The symmetric tensor proof is trivial after achieving the above argument.
\end{proof}

\section{Robustness of Our Algorithm under Noise}\label{sec:noise}
Let $\mytensor{T}$ be the true tensor, $\noise{\mytensor{T}} = \mytensor{T} + \Phi$ be the observed noisy tensor, where $\Phi$ is the noise. Let $\mymatrix{M}$ and $\noise{\mymatrix{M}}$ be the matrix prepared from $\mytensor{T}$ and $\noise{\mytensor{T}}$ by Procedure 2 for matrix subspace iteration.

\subsection{Perturbation Bounds}
\begin{lemma}[Perturbation in slice-based initialization step]
  \label{lemma:init_perturb}
  \begin{equation}
      \|\noise{\mymatrix{M}} - \mymatrix{M}\|_{\op} \leq 2 \|\lambda\| \|\Phi\|_{\op} + d\|\Phi\|_{\op}^2 
  \end{equation}
\end{lemma}

\begin{proof}
  \begin{align}
    \|\noise{\mymatrix{M}} - \mymatrix{M}\|_{\op} &\leq 2 \|\sum_{u=1}^d \mytensor{T}(\identitymatrix, \identitymatrix, \myvector{e}_u) \Phi(\identitymatrix, \identitymatrix, \mv{e}_u)^\top\|_{\op} + \|\sum_{u=1}^d\Phi(\identitymatrix, \identitymatrix, \mv{e}_u)\Phi(\identitymatrix, \identitymatrix, \mv{e}_u)^\top\|_{\op}\\
  \end{align}
    Let $\mm{E}_1:=\sum_{u=1}^d \mytensor{T}(\identitymatrix,\identitymatrix, e_u)\Phi(\identitymatrix, \identitymatrix, \mv{e}_u)^\top$ and $\mm{E}_2:=\sum_{u=1}^d\Phi(\identitymatrix, \identitymatrix, \mv{e}_u)\Phi(\identitymatrix, \identitymatrix, \mv{e}_u)^\top$ respectively. We have:
  \begin{align}
      \mm{E}_1 &= \sum_{r=1}^{R} \lambda_r \mv{a}_r \otimes \Phi(\identitymatrix, \mv{b}_r, \mv{c}_r)
  \end{align}
  Then $\forall \mv{x}, \mv{y} \in \R^d$,  
  \begin{align}
      \mv{x}^\top \mm{E}_1 \mv{y} &= \sum_{r=1}^R \lambda_r\mv{a}_r^\top\mv{x} \Phi(\mv{y}, \mv{b}_r, \mv{c}_r)\\
      & \leq (\sum_{r=1}^R \lambda_r \mv{a}_r^\top \mv{x})\|\Phi\|_{\op} \|\mv{y}\|\|\mv{b}_r\|\|\mv{c}_r\| \\
  \end{align}
    Since $\{\mv{a}_r\}_{r=1}^R$ are orthogonal, $\forall \mv{x}\in\R^d, \exists \mv{x}^\prime\in\R^R$ such that $x^\prime_r = \mv{a}_r^\top\mv{x}$ and $\|\mv{x}^\prime\|\leq \|\mv{x}\|$. Thus
  \begin{align}
      \mv{x}^\top \mm{E}_1 \mv{y} & \leq \|\Phi\|_{\op} \sum_{r=1}^R \lambda_r x^\prime_r \|\mv{y}\|
       \leq \|\Phi\|_{\op} \|\mv{\lambda}\| \|\mv{x}\| \|\mv{y}\|
  \end{align}
    For $\mm{E}_2$ (which is a symmetric matrix),
  \begin{align}
      \mv{x}^\top \mm{E}_2 \mv{x} &= \sum_{u=1}^d \|\Phi(\mv{x}, \identitymatrix, \mv{e}_u)\|^2 \leq d\|\Phi\|_{\op}^2\|\mv{x}\|^2
  \end{align}
\end{proof}
That is, $\|\mm{E}_1\|\leq \|\Phi\|_{\op} \|\mv{\lambda}\|$, and $\|\mm{E}_2\|\leq d\|\Phi\|_{\op}^2$.
  
\begin{lemma}[Perturbation in initialization step for symmetric case]
  \label{lemma:init_perturb_sym}
    For symmetric orthogonal tensor, for the matrix generated with trace-based initialization procedure for matrix subspace iteration of the first component, there exists $\{\lambda^\prime_r\}_{r=1}^R$ satisfies the following:
  \begin{equation}
      \noise{\mm{M}} = \sum_{r=1}^R \lambda^\prime_r \mv{a}_r\otimes \mv{a}_r + \Phi_M 
  \end{equation}
    and 
  \begin{equation}
    \|\Phi_M\|_{\op} \leq \|\mv{\lambda}\|\|\Phi\|_{\op} + d^{3/2}\|\Phi\|^2_{\op}.
  \end{equation}
\end{lemma}

\begin{proof}
    By the linearity of trace and tensor operators, we have the following results:
    \begin{align}
        \noise{\mm{M}} = \mt{T}(\identitymatrix, \identitymatrix, \mv{v}) + \mt{T}(\identitymatrix, \identitymatrix, \mv{v}_\phi) + \Phi(\identitymatrix, \identitymatrix, \mv{v}) + \Phi(\identitymatrix, \identitymatrix, \mv{v}_\phi)\label{line:noisy_sym_init} 
    \end{align}
    where
    \begin{align}
        (\mv{v})_k &= \Tr(\mt{T}(\identitymatrix, \identitymatrix, \mv{e}_k)) = \sum_{i=1}^d \sum_{r=1}^R \lambda_r (a_{ir})^2 a_{kr} = \sum_{r=1}^R\lambda_r a_{kr}\\
        (\mv{v}_\phi)_k &= \Tr(\Phi(\identitymatrix, \identitymatrix, \mv{e}_k))
    \end{align}

    First we notice that $\|\mv{v}_\phi\|$ is upperbounded:
    \begin{align}
    \|\mv{v}_\phi\|^2 = \sum_{k=1}^d \Tr^2(\Phi(\identitymatrix, \identitymatrix, \mv{e}_k)) 
        \leq \sum_{k=1}^d (d\|\Phi(\identitymatrix, \identitymatrix, \mv{e}_k)\|)^2_{\op}
        \leq d^3\|\Phi\|_{\op}^2
    \end{align}
    Similarly
    \begin{align}
      \|\mv{v}\|^2
    = \sum_{k=1}^d (\sum_{r=1}^R\lambda_r a_{kr})^2
    = \sum_{k=1}^d\sum_{\rho,r}^R\lambda_{\rho}\lambda_r a_{kr}a_{k\rho} 
    = \sum_{r,\rho}\lambda_r\lambda_{\rho}\mv{a}_r^\top\mv{a}_{\rho}
    = \sum_{r=1}^R\lambda_r^2
    \end{align}
    Thus the last two operator norm of terms of Eqn (\ref{line:noisy_sym_init}) can be bounded by 
    $$\|\Phi\|_{\op} (\|\mv{v}\|+\|\mv{v}_\phi\|) \leq \|\lambda\|\|\Phi\|_{\op} + d^{3/2}\|\Phi\|_{\op}^2$$
    
    The second term of Eqn (\ref{line:noisy_sym_init}) has the following form
    \begin{align}
        \mt{T}(\identitymatrix, \identitymatrix, \mv{v}_\phi) = \sum_{r=1}^R \lambda_r \mv{c}^\top_r \mv{v}_\phi \mv{a}_r\otimes\mv{a}_r 
    \end{align}
    Thus $\exists \mv{x}\in\R^R:\|\mv{x}\|\leq 1$, such that
    $\lambda^\prime_r = \lambda_r^2 + \lambda_r \mv{x}_r \|\mv{v}_\phi\|$
\end{proof}

\begin{lemma}[Perturbation in convergence step]
  \label{lemma:converge_perturb}
  \begin{align}
      \|\mmCols{A}{r}^\top \Phi_{(1)} (\mm{Q}^{(k)}_{\mmCols{C}{r}}\odot \mymatrix{Q}^{(k)}_{\mmCols{B}{r}})\|_{\op} \leq \sqrt{r} \|\Phi\|_{\op}\\
      \|(\mmCols{A}{r}^c)^{\top} \Phi_{(1)} (\mymatrix{Q}^{(k)}_{\mmCols{C}{r}}\odot \mymatrix{Q}^{(k)}_{\mmCols{B}{r}})\|_{\op} 
    \leq \sqrt{r} \|\Phi\|_{\op}
  \end{align}
\end{lemma}
\begin{proof}
  \begin{align}
      (\mmCols{A}{r}^\top \Phi_{(1)} (\mymatrix{Q}_{\mmCols{C}{r}}\odot \mymatrix{Q}_{\mmCols{B}{r}}))_{ij} = \sum_{(k,z,u)\in[d]^{\times 3}} \Phi_{kzu} (\mmCols{A}{r})_{ki} (\mm{Q}_{\mmCols{B}{r}})_{zj} (\mm{Q}_{\mmCols{C}{r}})_{uj}
  \end{align}
  $\forall \myvector{x}, \myvector{y}\in \R^r$ such that $\|\myvector{x}\|, \|\myvector{y}\|\leq 1$:
  \begin{align}
      \myvector{x}^\top (\mmCols{A}{r}^\top \Phi_{(1)} (\mm{Q}_{\mmCols{C}{r}}\odot\mm{Q}_{\mmCols{B}{r}})) \myvector{y}  
      &= \sum_{i,j\in [r]^{\times 2}}x_iy_j\sum_{k,z,u\in[d]^{\times 3}} \Phi_{kzu} (\mmCols{A}{r})_{ki} (\mm{Q}_B)_{zj} (\mm{Q}_C)_{uj}\\
      &= \sum_{j\in[r]} \Phi(\mmCols{A}{r}\myvector{x},(\mm{Q}_{\mmCols{B}{r}})_j,(\mm{Q}_{\mmCols{C}{r}})_j) y_j\\
  \end{align}
  By the definition of tensor operator norm, we have that $\forall 1\leq j\leq r$:
  \begin{align}
    \Phi(\mmCols{A}{r}\myvector{x},(\mm{Q}_{\mmCols{B}{r}})_j,(\mm{Q}_{\mmCols{C}{r}})_j) 
    &\leq \|\Phi\|_{\op}\|\mmCols{A}{r}\mv{x}\|\|(\mm{Q}_{\mmCols{B}{r}})_j\|\|(\mm{Q}_{\mmCols{C}{r}})_j\|\\
    &\leq \|\Phi\|_{\op} \|\mmCols{A}{r}\|_{\op}\|\mv{x}\|\\
    &= \|\Phi\|_{\op}\|\mv{x}\|
  \end{align}
  Thus
  \begin{align}
      \myvector{x}^\top (\mmCols{A}{r}^\top \Phi_{(1)} (\mm{Q}_{\mmCols{C}{r}}\odot\mm{Q}_{\mmCols{B}{r}})) \myvector{y}  
      &\leq \|\Phi\|_{\op}\|\mv{x}\| \sum_{j=1} y_j\\
      &\leq \|\myvector{y}\|_{1} \|\Phi\|_{\op} \|\myvector{x}\|\\
      &\leq \sqrt{r}\|\Phi\|_{\op}
  \end{align}
  The proof for 
      $\mmCols{A}{r}^\top \Phi_{(1)} (\mm{Q}_{\mmCols{C}{r}}\odot\mm{Q}_{\mmCols{B}{r}})$ is similar.  

\end{proof}

\subsection{Proof of Theorem~\ref{thm:robust}}
We prove the theorem by examine the success and convergence rate of the initialization stage (lemma \ref{lemma:init_step}) and the convergence stage (lemma \ref{lemma:convergence_step}). 

We first provide a few facts that will be used in the proofs.
\begin{myfact}
  \label{fact:convex} The convex combination of scalars is smaller than the largest scalar. That is, 
  $\forall \alpha\in [0,1]$:
  \begin{align}
    \alpha x_1 + (1 - \alpha) x_2 \leq \max\{x_1, x_2\}
  \end{align}
\end{myfact}
\begin{myfact}
  \label{fact:exp}
  For all $\theta \in (0,1)$ and $A,B\geq 0$:
  \begin{align}
    \frac{A}{A+\theta B} \leq \frac{1}{1+\theta \frac{B}{A}} \leq \frac{1}{(1+\frac{B}{A})^\theta} = (\frac{A}{A+B})^\theta 
  \end{align}
\end{myfact}
\begin{lemma}[Initialization step for noisy tensors]
  \label{lemma:init_step}
  If the operator norm of the noise tensor is bounded in the following way with a small enough constant $\delta_0$: 
  \begin{align}
        \|\Phi\|_{\op} \leq \min\{\delta_0\frac{\lambda_r^2-\lambda_{r+1}^2}{8\|\mv{\lambda}\|}, \sqrt{\delta_0}\frac{\lambda_r -\lambda_{r+1}}{2\sqrt{d}}\}
  \end{align}
  Then with probability $1-\mathcal{O}(\delta_0)$ matrix subspace iteration procedure yields a $r$-sufficient initialization in $\mathcal{O}(1)$ time. To be more specific, the tangent value of the subspace angle converges with a rate $|\frac{\lambda_{r+1}}{\lambda_r}|$.
\end{lemma}

\begin{proof}
    For matrix subspace iteration of 
    $\noise{\mm{M}} = \mm{M} + \Phi_{\mm{M}} = \mm{A}\mm{D}\mm{A}^\top + \Phi_{\mm{M}}$, we have the following:
    \begin{align}
        t^{(k+1)}_{\mmCols{A}{r}} &\leq \frac
        {\sigma_{\max}((\mmCols{A}{r}^c)^\top\mm{A}D\mm{A}^\top \mm{Q}_{\mmCols{A}{r}})+\sigma_{\max}((\mmCols{A}{r}^c)^\top \Phi_{\mm{M}}\mm{Q}^{(k)}_{\mmCols{A}{r}})}
        {\sigma_{\min}(\mmCols{A}{r}^\top\mm{A}D\mm{A}^\top \mm{Q}_{\mmCols{A}{r}})-\sigma_{\max}(\mmCols{A}{r}^\top \Phi_{\mm{M}}\mm{Q}^{(k)}_{\mmCols{A}{r}})}
        \\
        &\leq \frac{d_{r+1}\sin \theta_A^{k} + \|\Phi_{\mm{M}}\|_{\op}}
        {d_r \cos\theta_A^{k} - \|\Phi_{\mm{M}}\|_{\op}}
    \end{align}
    where $\theta_A^k$ is the principle angle between the subspace spanned by $\mmCols{A}{r}$ and $\mm{Q}^{(k)}_{\mmCols{A}{r}}$, and $t_{\mmCols{A}{r}}^{(k)}$ is $\tan \theta_A^k$.\\
    Let $u$ denote 
    $\frac{\|\Phi_{\mm{M}}\|_{\op}}{\gap_r^\prime\cos\theta_A^k}$, where $\gap^\prime_r:=d_r - d_{r+1}$. We have:

    \begin{align}
        t^{(k+1)}_{\mmCols{A}{r}} &\leq \frac
        {d_{r+1}\sin\theta^k_A + u\gap_r^\prime \cos\theta^k_A}
        {d_{r}\cos\theta^k_A - u\gap_r^\prime \cos\theta^k_A}
        \\
        &\leq \frac
        {d_{r+1}}{d_r - u\gap^\prime_r} t^{(k)}_{\mmCols{A}{r}}
        +\frac{u\gap_r^\prime}{d_r- u\gap^\prime_r}
        \\
        &= \frac{d_r-2u\gap^\prime_r}{d_r-u\gap^\prime_r}\cdot\frac
        {d_{r+1}}{d_r - 2u\gap^\prime_r} t^{(k)}_{\mmCols{A}{r}}
        +\frac{u\gap_r^\prime}{d_r- u\gap^\prime_r}\cdot 1
        \\
        &\leq \max\{\frac{d_{r+1}}{d_r - 2u\gap^\prime_r} t^{(k)}_{\mmCols{A}{r}},1\} \qquad\text{(By Fact~\ref{fact:convex})}
        \\
        &= \max\{\frac{d_{r+1}}{d_{r+1} + (1-2u)\gap^\prime_r} t^{(k)}_{\mmCols{A}{r}},1\}
        \\
        &\leq \max\{(\frac{d_{r+1}}{d_{r}})^\theta t^{(k)}_{\mmCols{A}{r}}, 1\} \qquad\text{(By Fact~\ref{fact:exp})}
    \end{align}
    where $\theta := 1 - 2u \leq 1$.
    Since $\Pr\{\cos\theta^0_A>0\} = 1$, by bounding $\|\Phi_{\mm{M}}\|_{\op} \leq \frac{d_r-d_{r+1}}{2}\delta_0$ with small enough constant $\delta_0$, combined with Proposition $B.2$ in \cite{wang2017tensor},  we can verify that $2u \leq 1$ with probability 1 - $\mathcal{O}(\delta_0)$. It is worth noticing that in the noiseless case, we can find a good initialization for matrix subspace iteration with probability $1$.

    By lemma \ref{lemma:init_perturb}, for the slice based initialization, $d_r = \lambda_r^2$, and $\|\Phi_{\mm{M}}\| \leq 2\|\lambda\|\|\Phi\|_{\op}$ $+$ $d\|\Phi\|^2_{\op}$, we have $1 - 2u \geq 0$ by bounding:
    \begin{align}
        \|\Phi\|_{\op} \leq \min\{\delta_0\frac{\lambda_r^2-\lambda_{r+1}^2}{8\|\mv{\lambda}\|}, \sqrt{\delta_0}\frac{\lambda_r -\lambda_{r+1}}{2\sqrt{d}}\}
    \end{align}
\end{proof}

\begin{lemma}[Convergence step for noisy tensors]
  \label{lemma:convergence_step}
  Assume we have the noise tensor bounded in operator norm such that:
    \begin{align}
      \|\Phi\|_{\op} \leq \frac{1}{2\sqrt{2}}\frac{\epsilon^\prime\gap_r}{\sqrt{r}}\label{line:loglog_cond1}
    \end{align}

    where $$\gap_r := \lambda_r - \lambda_{r+1}$$
    Then we have either $(1)$ $t_{\mmCols{A}{r}}$ is small enough:
    \begin{align}
        t^{(k+1)}_{\mmCols{A}{r}} \leq \epsilon^\prime
    \end{align} 
    Or $(2)$ converges by the following rule:

    \begin{align}
        t^{(k+1)}_{\mmCols{A}{r}} \leq (\frac{\lambda_{r+1}}{\lambda_{r}})^\theta t^{(k)}_{\mmCols{B}{r}} t^{(k)}_{\mmCols{C}{r}}
    \end{align} 
    where
    \begin{equation}\theta := 1 - \frac{2}{\gap_r}(\frac{\sqrt{2}}{\epsilon^\prime} + 1) \sqrt{r} \|\Phi\|_{\op}\end{equation}
\end{lemma}
\begin{proof}
  The proof for Theorem~\ref{lemma:convergence_step} follows the same style of Lemma $B.1$ in \cite{wang2017tensor}. Similar to the noiseless case, we have: 
    \begin{align}
        t^{(k+1)}_{\mmCols{A}{r}} &\leq \frac
        {\lambda_{r+1} \sin\theta^{(k)}_B\sin\theta^{(k)}_C + \sigma_{\max}((\mmCols{A}{r}^c)^\top \Phi_{(1)} (\mymatrix{Q}^{(k)}_{C_{(r)}}\odot \mymatrix{Q}^{(k)}_{\mymatrix{B}_{(r)}}))}
        {\lambda_{r} \cos\theta^{(k)}_B\cos\theta^{(k)}_C - \sigma_{\max}(\mmCols{A}{r}^\top \Phi_{(1)} (\mymatrix{Q}^{(k)}_{C_{(r)}}\odot \mymatrix{Q}^{(k)}_{\mymatrix{B}_{(r)}}))}
    \end{align}
    where $\theta_U^k$ is the principle angle between the subspace spanned by $\mmCols{U}{r}$ and $\mm{Q}^{(k)}_{\mmCols{U}{r}}$ for $U\in\{A,B.C\}$, and $t_{\mmCols{A}{r}}^{(k)}$ is $\tan \theta_A^k$.
    Let $\sigma$ denote  the maximum of $\sigma_{\max}$($(\mmCols{A}{r}^c)^\top$ $\Phi_{(1)}$ $(\mymatrix{Q}^{(k)}_{C_{(r)}}\odot \mymatrix{Q}^{(k)}_{\mymatrix{B}_{(r)}})$) and 
    $\sigma_{\max}(\mmCols{A}{r}^\top \Phi_{(1)} (\mymatrix{Q}^{(k)}_{C_{(r)}}\odot \mymatrix{Q}^{(k)}_{\mymatrix{B}_{(r)}}))$, and let $r_{1} := \frac{\sqrt{2}\sigma}{\epsilon^\prime \gap_r}$, $r_{2} := \frac{2\sigma}{\gap_r}$. Thus
    \begin{align}
        t^{(k+1)}_{\mmCols{A}{r}} &\leq \frac
        {\lambda_{r+1} \sin\theta^{(k)}_B \sin\theta^{(k)}_C + \sigma}
        {\lambda_{r} \cos\theta^{(k)}_B\cos\theta^{(k)}_C - \sigma}\\
        &= \frac
        {\lambda_{r+1} \sin\theta^{(k)}_B\sin\theta^{(k)}_C + r_1\gap_r \epsilon^\prime \frac{\sqrt{2}}{2}}
        {\lambda_{r} \cos\theta^{(k)}_B\cos\theta^{(k)}_C - \frac{1}{2}r_2 \gap_r}\\
    \end{align}
    For bounded $\theta^{(k)}_B$ and $\theta^{(k)}_C$ such that $\tan\theta^{(k)}_B$ and $\tan\theta^{(k)}_C$ are less than $1$, we have $\cos(\theta^{(k)}_B - \theta^{(k)}_C)\geq \frac{\sqrt{2}}{2}$, 
    and $\cos\theta^{(k)}_B\cos\theta^{(k)}_C \geq \frac{1}{2}$. Thus
    \begin{align}
        t^{(k+1)}_{\mmCols{A}{r}} 
        &\leq \frac
        {\lambda_{r+1} \sin\theta^{(k)}_B\sin\theta^{(k)}_C 
        + 
        r_1\gap_r\epsilon^\prime \cos(\theta^{(k)}_B - \theta^{(k)}_C)}
        {\lambda_{r} \cos\theta^{(k)}_B\cos\theta^{(k)}_C - r_2 \gap_r\cos\theta^{(k)}_B\cos\theta^{(k)}_C}\\
        &= \frac
        {\lambda_{r+1} + r_1\gap_r\epsilon^\prime}
        {\lambda_r-r_2\gap_r} \frac{\sin\theta^{(k)}_B\sin\theta^{(k)}_C}{\cos\theta^{(k)}_B\cos\theta^{(k)}_C}
        + \frac{r_1\gap_r}{\lambda_r - r_2\gap_r}\epsilon^\prime\\
        &= \frac
        {\lambda_{r+1} + r_1\gap_r\epsilon^\prime}
        {\lambda_r-r_2\gap_r} t_{\mmCols{B}{r}}^{(k)}t_{\mmCols{C}{r}}^{(k)}
        + \frac{r_1\gap_r}{\lambda_r - r_2\gap_r}\epsilon^\prime\\
        &= \frac
        {\lambda_{r+1} + r_1\gap_r\epsilon^\prime}
        {\lambda_{r+1}+ (1-r_2)\gap_r} t_{\mmCols{B}{r}}^{(k)}t_{\mmCols{C}{r}}^{(k)}
        + \frac{r_1\gap_r}{\lambda_{r+1} + (1 - r_2)\gap_r}\epsilon^\prime\\
        &= (1-\alpha)
        \frac
        {\lambda_{r+1} + r_1\gap_r\epsilon^\prime}
        {\lambda_{r+1}+ (1-r_1-r_2)\gap_r} t_{\mmCols{B}{r}}^{(k)}t_{\mmCols{C}{r}}^{(k)}
        + 
        \alpha \epsilon^\prime
    \end{align}
    where 
    \begin{align}
        \alpha = \frac{r_1\gap_r}{\lambda_{r+1} + (1 - r_2)\gap_r}
    \end{align}
    Thus 
    \begin{align}
        t^{(k+1)}_{\mmCols{A}{r}}
        &\leq \max\{\frac{\lambda_{r+1} + r_1\gap_r\epsilon^\prime}{\lambda_{r+1}+ (1-r_1-r_2)\gap_r} t_{\mmCols{B}{r}}^{(k)}t_{\mmCols{C}{r}}^{(k)}, \epsilon^\prime\}
    \end{align}

    Similarly,
    \begin{align}
        \frac{\lambda_{r+1} + r_1\gap_r\epsilon^\prime}{\lambda_{r+1}+ (1-r_1-r_2)\gap_r} 
        &= (1-\beta) \frac{\lambda_{r+1}}{\lambda_{r+1}+(1-2r_1 - r_2)\gap_r} + \beta \epsilon^\prime\\ 
    \end{align}
    where 
    \begin{align}
        \beta = \frac{r_1\gap_r}{\lambda_{r+1} + (1 - r_1 - r_2)\gap_r}
    \end{align}
    
    Let $\theta$ denote $1-2r_1 - r_2$. As long as $\theta > 0$ (that is, ),
    \begin{align}
        t^{(k+1)}_{\mmCols{A}{r}} &\leq 
        \max\{\max\{\frac{\lambda_{r+1}}{\lambda_{r+1} + \theta \gap_r}, \epsilon^\prime\}t_{\mmCols{B}{r}}^{(k)}t_{\mmCols{C}{r}}^{(k)} ,\epsilon^\prime\}\\
        &\leq \max\{\max\{(\frac{\lambda_{r+1}}{\lambda_{r}})^\theta, \epsilon^\prime\} t_{\mmCols{B}{r}}^{(k)}t_{\mmCols{C}{r}}^{(k)} ,\epsilon^\prime\}
    \end{align}

    If $\epsilon^\prime \geq (\frac{\lambda_{r+1}}{\lambda_r})^\theta$ and the $r-$sufficient condition is met, 
    $$t^{(1)}_{\mmCols{A}{r}} \leq \epsilon^\prime t^{(0)}_{\mmCols{B}{r}} t^{(0)}_{\mmCols{C}{r}} \leq \epsilon^\prime$$

    Either the convergence requirement is met after the first iteration, or the procedures converges following:
    \begin{align}
        t^{(k+1)}_{\mmCols{A}{r}} &\leq 
        (\frac{\lambda_{r+1}}{\lambda_{r}})^\theta t_{\mmCols{B}{r}}^{(k)}t_{\mmCols{C}{r}}^{(k)}
    \end{align}

    Combined with lemma \ref{lemma:converge_perturb}, the condition $\theta > 0$ is equivalent to:
    \begin{align}
        \|\Phi\|_{\op} \leq \frac{\gap_r}{2(\frac{\sqrt{2}}{\epsilon^\prime}+1)\sqrt{r}}\label{line:loglog_cond}
    \end{align}
    Condition (\ref{line:loglog_cond}) is satisfied $\forall 1\leq r\leq R$ as long as:
    \begin{align}
      \|\Phi\|_{\op} \leq \frac{1}{2(\frac{\sqrt{2}}{\epsilon^\prime}+1)}\frac{\min_r \gap_r}{\sqrt{R}}
    \end{align}

\end{proof}



\end{document}